\documentclass[11pt]{article}
\usepackage{yandex-research}

\usepackage[utf8]{inputenc} 
\usepackage[T1]{fontenc}    
\usepackage{hyperref}       
\usepackage{url}            
\usepackage{fontawesome5}   
\usepackage{booktabs}       
\usepackage{amsfonts}       
\usepackage{nicefrac}       
\usepackage{microtype}      
\usepackage{bbm}           
\usepackage{xcolor}         
\usepackage{amsmath}
\usepackage{graphicx}
\usepackage{multirow}
\usepackage{wrapfig}
\usepackage[skins,breakable]{tcolorbox}
\usepackage{textcomp}
\usepackage{xcolor}

\tcbuselibrary{listings}
\usepackage{xcolor}
\usepackage{textcomp}
\usepackage[skins,breakable]{tcolorbox}
\tcbuselibrary{listings}

\definecolor{gsmframeheader}{RGB}{128,128,128}
\definecolor{gsmframebackground}{RGB}{245,245,245}
\definecolor{gsmerrorbackground}{RGB}{255,225,225}
\definecolor{gsmerrortext}{RGB}{155,30,30}
\definecolor{gsmcorrectbackground}{RGB}{222,242,225}
\definecolor{gsmcorrecttext}{RGB}{28,100,49}

\newtcolorbox{gsm8kframe}[3]{
  enhanced,
  breakable,
  colback=gsmframebackground,
  colframe=gsmframeheader,
  colbacktitle=gsmframeheader,
  coltitle=white,
  boxrule=0.5pt,
  arc=1.5mm,
  boxsep=0pt,
  left=5pt, right=5pt, top=3pt, bottom=5pt,
  toptitle=3pt, bottomtitle=3pt,
  titlerule=0pt,
  fontupper=\rmfamily\small,
  fonttitle=\rmfamily\small\bfseries,
  title={Sampling Steps: #1\hfill{\mdseries Accuracy: #2}},
  before upper={%
    \setlength{\parindent}{0pt}%
    \setlength{\parskip}{0pt}%
    \textbf{Question:} #3\par\smallskip
    \textbf{Generated answer:}\par\nobreak
  },
  before skip=8pt,
  after skip=8pt
}

\lstdefinestyle{gsmcode}{
  language=Python,
  basicstyle=\ttfamily\footnotesize,
  keywordstyle=\color{black},
  commentstyle=\color{black},
  stringstyle=\color{black},
  identifierstyle=\color{black},
  columns=fullflexible,
  keepspaces=true,
  tabsize=4,
  showspaces=false,
  showstringspaces=false,
  showtabs=false,
  showlines=true,
  upquote=true,
  numbers=none,
  breaklines=true,
  breakatwhitespace=false,
  breakautoindent=true,
  breakindent=1em,
  postbreak={},
  aboveskip=0pt,
  belowskip=0pt,
  mathescape=false,
  texcl=false,
  escapeinside={(*@}{@*)},
  escapebegin=\ttfamily\footnotesize,
  literate={’}{{\textquoteright}}1
           {‘}{{\textquoteleft}}1
           {“}{{\textquotedblleft}}1
           {”}{{\textquotedblright}}1
           {–}{{\textendash}}1
           {—}{{\textemdash}}1
           {…}{{\textellipsis}}1
           {−}{{\textminus}}1
           {×}{{\texttimes}}1
           {÷}{{\textdiv}}1
           {€}{{\texteuro}}1
           {£}{{\textsterling}}1
           { }{{~}}1
}

\newtcblisting{gsmcode}[4][]{
  enhanced,
  breakable,
  listing only,
  listing options={style=gsmcode},
  colback=gsmframebackground,
  colframe=gsmframeheader,
  colbacktitle=gsmframeheader,
  coltitle=white,
  boxrule=0.5pt,
  arc=1.5mm,
  boxsep=0pt,
  left=5pt, right=5pt, top=3pt, bottom=5pt,
  toptitle=3pt, bottomtitle=3pt,
  titlerule=0pt,
  fonttitle=\rmfamily\small\bfseries,
  title={Sampling Steps: #2\hfill{\mdseries Accuracy: #3\quad Executed answer: #4}},
  before skip=8pt,
  after skip=8pt,
  #1
}

\usepackage{algorithm}
\usepackage{listings}
\usepackage{caption}
\usepackage{mathtools}

\definecolor{analysisblue}{RGB}{29,117,170}
\definecolor{analysisorange}{RGB}{250,125,14}
\definecolor{analysispurple}{RGB}{120,100,230}

\usepackage{amsthm}

\newtheorem{proposition}{Proposition}

\definecolor{codegreen}{RGB}{42,125,128}
\definecolor{codeblue}{RGB}{0,45,180}
\definecolor{codepink}{RGB}{220,40,130}
\definecolor{codeorange}{RGB}{220,95,0}

\lstdefinestyle{pseudocode}{
    language=Python,
    basicstyle=\ttfamily\small,
    keywordstyle=\color{codeblue}\bfseries,
    commentstyle=\color{codegreen}\bfseries,
    stringstyle=\color{codeorange},
    showstringspaces=false,
    columns=fullflexible,
    keepspaces=true
}

\title{Revisiting Tabular Deep Regression: \\ Objectives and the Limits of Global \\ Early Stopping}
\author{%
\textbf{Danil Rudenko}\textsuperscript{1,2,*},
\textbf{Georgii Bulgakov}\textsuperscript{1,2,*},
\textbf{Nikolay Kartashev}\textsuperscript{3,1,*}, \\
\textbf{Ivan Rubachev}\textsuperscript{1,3}, 
\textbf{Daniil Krasylnikov}\textsuperscript{1,\textdagger},
\textbf{Artem Babenko}\textsuperscript{1,3}
\affiliations{%
\textsuperscript{1}Yandex Research,\enspace
\textsuperscript{2}BRAInLab,\enspace
\textsuperscript{3}HSE University\enspace
}
}

\begin{document}
\maketitle
\begingroup 
\renewcommand{\thefootnote}{\fnsymbol{footnote}}
\footnotetext[1]{Equal contribution.}
\footnotetext[2]{Work done while at Yandex Research.}
\endgroup


\begin{abstract}
While deep learning models for tabular data have recently received significant research attention, the choice of objective functions for regression tasks remains an understudied degree of freedom. Several lines of research offer alternatives to the standard Mean
Squared Error (MSE), but differences in datasets, backbones, and
evaluation protocols make their practical benefits difficult to compare. In this work, we address this gap by performing a systematic re-evaluation of regression objectives across a broad collection of established tabular datasets and multiple neural backbones. Our comparison reveals that many proposed objectives provide limited or inconsistent gains over a properly tuned MSE baseline, while quantile loss yields consistent aggregate improvements over MSE across the tested backbones. We then identify a limitation that remains even with this strong objective: regions with different levels of target noise can reach their best predictive performance at different training epochs. Consequently, a single global early-stopping checkpoint can impose a compromise between uncertainty regions, even when the objective can represent heteroskedastic target distributions. Motivated by this limitation, we introduce a heteroskedasticity-aware extension of quantile loss that uses input-dependent label smoothing as a potential remedy and improves over quantile loss on several backbones. Overall, our results provide a modern empirical picture of regression objectives for tabular deep learning and highlight how differences in optimal stopping times across uncertainty regions can limit the effectiveness of global early stopping. The code is available at
\url{https://github.com/yandex-research/tabular-regression}.

\end{abstract}

\section{Introduction}

Despite the growing adoption of deep learning models for tabular data, in practice, it is not entirely clear what objective function should be used to solve regression tasks with DNNs. While the design of advanced regression objectives is an active research direction \citep{imani2018improving,stewart2025building,farebrother2024stopregressingtrainingvalue,song2025decodingbased,zhu2025gradientalignedregressionpairwise,zhang2023improvingdeepregressionordinal,zha2023rank,zhang2025improvingdeepregressiontightness}, the traditional Mean Squared Error (MSE) remains widely used. Existing studies evaluate alternative objectives on different datasets, backbones, and tuning protocols, making it difficult to establish which alternatives provide reliable improvements in practice.

We revisit traditional and modern regression objectives on 41
established tabular datasets with extensive hyperparameter tuning,
focusing on conventional DNNs (pretrained PFNs are
outside our scope). We find that most objectives do not consistently improve over a
carefully tuned MSE baseline. Two approaches stand out on MLP:
Gaussian-smoothed binning objectives
\citep{imani2018improving,stewart2025building}, which represent targets
as distributions over discrete bins, and quantile loss
\citep{Koenker.Bassett1978}, which predicts conditional quantiles of
$p(y\mid x)$ via the pinball objective. Of these, quantile loss yields
more consistent gains across different backbones. Related evidence comes from the
PFN literature: TabICLv2 \citep{tabiclv2} adopts quantile loss for
regression and reports benefits over MSE and binning; our findings
complement this with a systematic comparison in conventional tabular
models.

Having identified quantile loss as a strong regression objective,
we investigate a limitation that remains under standard training:
regions with different levels of data uncertainty can reach their
best predictive performance at different epochs. Consequently, a single early-stopping checkpoint can impose a compromise between uncertainty regions. Such input-dependent variation in target noise is common in tabular
data \citep{kartashev2025unveiling}. We examine this conflict through controlled synthetic experiments and uncertainty-based checkpoint selection on real datasets.

Motivated by this limitation, we propose \textit{Adaptive Quantile Regression} (AdaQR), which incorporates jointly learned, input-dependent smoothing into quantile training. Effectively, by reducing the sensitivity of quantile updates to individual label realizations in high-uncertainty regions, adaptive smoothing aims to delay overfitting there while allowing learning to continue in lower-uncertainty regions. AdaQR is designed as one potential remedy for the mismatch induced by a shared stopping point, which improves upon standard quantile loss on several DNN backbones as we show in our experiments.

Our main contributions are:
\begin{itemize}
    \item We benchmark 18 regression objectives on MLP across 41 datasets and evaluate selected objectives on three advanced tabular DNN backbones, using up to 100 hyperparameter tuning trials and the Muon optimizer \citep{jordan2024muon,gorishniy2026benchmarking}.

    \item We show that most objectives do not consistently beat tuned MSE. Gaussian-smoothed binning objectives and quantile loss stand out on MLPs, while quantile loss provides more consistent gains across the tested backbones.

    \item We identify a limitation of global early stopping in tabular
regression: synthetic experiments reveal different preferred
stopping epochs across noise levels, while real-data experiments show
that selecting checkpoints separately for cohorts defined by estimated
uncertainty can improve predictive performance over global selection.

    \item We propose a new AdaQR objective that provides a jointly learned smoothing mechanism for investigating how uncertainty-aware training can mitigate stopping conflicts and evaluate this objective on established tabular backbones.
\end{itemize}

\section{Related Work}
\label{sec:related_work}

\textbf{Regression objectives.}
Alternatives to MSE include robust losses \citep{huber1964robust}, target discretization \citep{imani2018improving,stewart2025building,farebrother2024stopregressingtrainingvalue}, conditional distribution modeling \citep{bishop1994mixture,Koenker.Bassett1978,song2025decodingbased,han2022card}, and representation regularization \citep{zha2023rank,zhang2025improvingdeepregressiontightness,zhang2023improvingdeepregressionordinal}. We compare representative methods under a common tabular evaluation protocol and describe the individual objectives in the following sections.

\textbf{Early stopping and overfitting.}
Different parts of a model or subsets of the data can reach their best generalization performance at different training epochs. \citet{heckel2020early} analyze different bias--variance tradeoffs across feature directions and network layers, showing that adjusting learning rates can improve early-stopping performance. \citet{bao2023token} observe that some token groups are already overfitting while others remain underfit at the selected checkpoint in sequence-to-sequence models. \citet{berta2025rethinking} show that calibration and refinement errors are minimized at different stages of training, including in tabular classification. We extend the uncertainty-based analysis of tabular models
\citep{kartashev2025unveiling} to training trajectories and checkpoint
selection, examining how input-dependent noise relates to the epochs
at which different regions achieve their best predictive performance.

\textbf{Heteroskedasticity-aware objectives and adaptive smoothing.}
A large body of work models input-dependent uncertainty in regression \citep{kendall2017uncertainties,immer2023effective}. However, joint mean and variance estimation can introduce optimization difficulties \citep{seitzer2022on,wong2023understanding}, motivating methods that preserve the mean-prediction accuracy of a corresponding MSE-trained model \citep{stirn2023faithful}. Related work on quantile regression investigates smoothing and local scale adaptation. \citet{tan2022high} study convolution smoothing of the pinball loss to improve optimization and statistical estimation. \citet{fasiolo2021fast} derive covariate-dependent smoothing scales
from a preliminary location--scale fit, while \citet{jung2015efficient}
estimate local scales using preliminary quantile or absolute-residual
regressions. In distributional reinforcement learning, \citet{malekzadeh2024robust} propose a noise-motivated generalization of quantile-Huber loss with adaptive parameter adjustment. AdaQR combines an energy-distance objective with an input-dependent
smoothing head learned jointly with the neural quantile predictor.
We investigate this formulation through its effects on
uncertainty-dependent training dynamics and tabular point prediction.

\section{Revisiting regression objectives}
\label{sec:revisiting}

\textbf{Notation.}
We consider supervised regression on tabular data.
Let $D = \{(x_i, y_i)\}_{i=1}^{n}$ denote a dataset, where
$x_i \in \mathbb{R}^d$ is the feature vector and
$y_i \in \mathbb{R}$ is the target.
We denote the neural model by
$f_\theta: \mathbb{R}^d \rightarrow \mathbb{R}^{d'}$,
where the output dimension $d'$ depends on the objective and
need not equal one.
For pointwise objectives, the loss on an observation is
$l(f_\theta(x_i), y_i)$.

\subsection{Losses}
  \label{sec:losses}

\begin{figure*}[t!]
\centering
\includegraphics[width=0.9\textwidth]{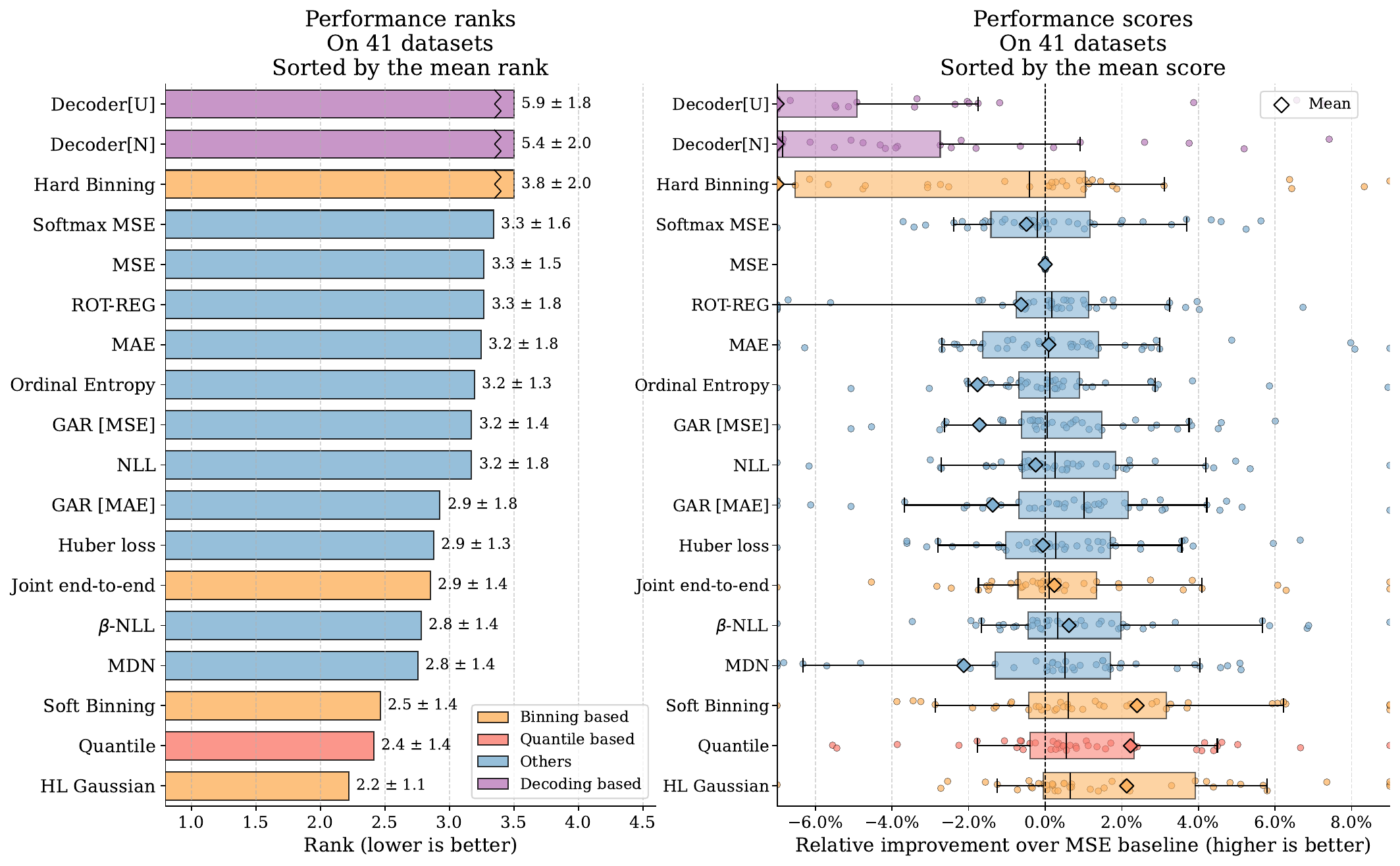}
\caption{Performance comparison of different objective functions using an MLP backbone across 41\\ BeyondArena datasets. (\emph{Left}) Average rank performance across all datasets. (\emph{Right}) Relative RMSE improvement over the standard MSE baseline, visualizing the magnitude of performance gains or degradation. The box plots visualize the 10th, 25th, 50th, 75th, and 90th percentiles of the jitter plots. Notation details are provided in the beginning of Section~\ref{sec:results}.}
\label{fig:vertical-revisiging}
\end{figure*}

  We compare the following regression objectives.

\textbf{Traditional losses.}
MSE penalizes squared errors, MAE penalizes absolute errors, and
Huber loss \citep{huber1992robust} combines quadratic penalties for
small errors with linear penalties for large ones.

\textbf{Binning-based losses.}
These methods discretize the target range into bins and train a
categorical prediction head to match a target distribution.
Hard Binning assigns each target to its nearest bin, while
Soft Binning distributes probability mass according to Gaussian
weights evaluated at bin centers \citep{stewart2025building}.
HL-Gaussian instead integrates a truncated Gaussian density over
each bin \citep{imani2018improving}.
Joint end-to-end encoding learns the target distribution using
an auxiliary encoder and a reconstruction objective
\citep{stewart2025building}.
All four methods return the predicted expectation over bin centers.
Softmax-MSE uses the same prediction head but directly minimizes
the squared error of this expectation \citep{stewart2025building}.

\textbf{Quantile regression.}
The model predicts a grid of conditional quantiles using the pinball
loss \citep{Koenker.Bassett1978}; their average serves as the
scalar prediction.

\textbf{Likelihood-based objectives.}
Gaussian NLL jointly learns a conditional mean and variance.
$\beta$-NLL reweights this objective by a detached power of the
predicted variance to address optimization difficulties
\citep{seitzer2022on}.
Mixture Density Networks model the conditional distribution as
a Gaussian mixture \citep{bishop1994mixture}.
The autoregressive approach of \citet{song2025decodingbased}
instead predicts tokenized targets using next-token cross-entropy.
We evaluate both its normalized and unnormalized target encodings.

\textbf{Additional objectives and regularizers.}
GAR combines pointwise and pairwise terms to capture target values
and their differences \citep{zhu2025gradientalignedregressionpairwise}.
Ordinal Entropy \citep{zhang2023improvingdeepregressionordinal}
and ROT-REG \citep{zhang2025improvingdeepregressiontightness}
augment regression training with representation regularization,
using entropy-based and optimal-transport-based terms, respectively.

We provide a detailed description of all objectives in
Appendix~\ref{app:losses}.

\subsection{Experiments}
\label{sec:revisiting_exps}

\textbf{Datasets.} We evaluate a subset of 41 regression datasets from BeyondArena \citep{purucker2026beyond}, excluding three computationally demanding datasets to keep hyperparameter tuning feasible within our 24-hour budget. We use the first predefined split of each retained dataset. Further details are provided in Appendix~\ref{app:datasets}.

\subsubsection{Implementation details}

\textbf{Model backbone.} We first compare regression objectives on the simplest MLP backbone, then evaluate selected objectives on additional advanced architectures. Depending on the objective, the prediction head outputs either a scalar, $\hat{y}_i = f_\theta(x_i) \in \mathbb{R}$, or a vector parameterizing the corresponding predictive output.

\textbf{Experimental setup.} We follow the evaluation protocol of \citet{gorishniy2025tabm},
detailed in Appendix~\ref{app:protocol}.
Unless stated otherwise, we use Muon-based training
\citep{jordan2024muon}, following recent evidence of its effectiveness
on tabular data \citep{gorishniy2026benchmarking}.
For each dataset and objective, we tune the model, objective, and
optimizer hyperparameters on the validation set.
We then retrain the selected configuration from scratch across
15 random seeds and report test RMSE.
Method comparisons use the mean and standard deviation across
these seeds, as described in Appendix~\ref{app:comparison}.

\subsubsection{Results}
\label{sec:results}

We evaluate all methods following the protocol above and report the results in Figure~\ref{fig:vertical-revisiging}. 
For readability, Decoder[U] and Decoder[N] denote the unnormalized and normalized decoders from \citet{song2025decodingbased}, respectively; GAR[MSE] and GAR[MAE] denote GAR variants with MSE- and MAE-based formulations \citep{zhu2025gradientalignedregressionpairwise}, respectively.

\textbf{Key takeaways.}
\begin{enumerate}
    \item Most objectives do not provide consistent improvements
    over a carefully tuned MSE.

    \item Soft Binning \citep{stewart2025building},
    HL-Gaussian \citep{imani2018improving}, and quantile loss
    \citep{Koenker.Bassett1978} stand out in the aggregate results,
    improving upon MSE in terms of both ranks and relative RMSE improvement.
\end{enumerate}

These findings also hold under AdamW
\citep{loshchilov2018decoupled}.
Table~\ref{tab:mlp_losses_vs_optimizers} shows that HL-Gaussian
and quantile loss improve over MSE in aggregate with this optimizer.
The Muon configurations achieve larger median gains over the
common MSE$+$AdamW reference, highlighting the importance of
optimizer choice when evaluating regression objectives.

\begin{table}[h]
\centering
\setlength{\tabcolsep}{2mm}
\caption{
Objective and optimizer comparison on MLP across 41 datasets.
Gain is the median relative RMSE improvement over the common
MSE$+$AdamW baseline; positive values indicate improvement.
W/T/L counts wins, ties, and losses under the standard-deviation-based
comparison rule in Appendix~\ref{app:comparison}.
Bold denotes the largest median gain within each optimizer.
}
\label{tab:mlp_losses_vs_optimizers}
\begin{tabular}{lcc@{\hspace{7mm}}cc}
\toprule
& \multicolumn{2}{c}{\textbf{AdamW}}
& \multicolumn{2}{c}{\textbf{Muon}} \\
\cmidrule(lr){2-3}\cmidrule(lr){4-5}
\textbf{Objective}
& \textbf{Gain (\%)}
& \textbf{W / T / L}
& \textbf{Gain (\%)}
& \textbf{W / T / L} \\
\midrule
MSE
& \textemdash & \textemdash
& +1.50 & 18 / 19 / 4 \\
HL-Gaussian
& \textbf{+0.78} & 19 / 19 / 3
& \textbf{+2.34} & 26 / 11 / 4 \\
quantile loss
& +0.67 & 15 / 23 / 3
& +2.18 & 24 / 13 / 4 \\
\bottomrule
\end{tabular}
\vspace{-10pt}
\end{table}

\subsubsection{Advanced backbones}
\label{sec:advanced_models}

Having identified promising objectives on MLP, we examine whether
their benefits transfer to other state-of-the-art backbone architectures.
We retain HL-Gaussian as a representative binning-based objective
and compare it and quantile loss against MSE.

Following the same evaluation protocol, we train TabPack with
periodic feature embeddings \citep{gorishniy2026tabpack},
TabM-mini with PLE-Q embeddings \citep{gorishniy2025tabm},
and RealMLP \citep{holzmuller2024realmlp}.
Implementation details are provided in
Appendix~\ref{app:advanced_models}.

Table~\ref{tab:quantile_backbones} shows that quantile loss provides
more consistent gains across the tested backbones than HL-Gaussian.
On TabPack, HL-Gaussian has more losses than wins against MSE,
despite a small positive median improvement; on RealMLP, its
median gain is also limited.
In contrast, quantile loss achieves positive median improvements
and more wins than losses on every tested backbone.
These results support quantile loss as a practical alternative
to MSE, while showing that its benefits remain dataset-dependent.

\begin{table}[t]
\centering
\setlength{\tabcolsep}{2mm}
\caption{
Quantile loss yields consistent aggregate improvements across
the tested backbones.
Each objective is compared against MSE on the same backbone.
Gain is the median relative RMSE improvement; positive values
indicate improvement.
W/T/L follows the comparison rule in Appendix~\ref{app:comparison}.
Bold denotes the larger median gain within each backbone,
with both values highlighted when equal at the reported precision.
}
\label{tab:quantile_backbones}
\begin{tabular}{lcc@{\hspace{7mm}}cc}
\toprule
& \multicolumn{2}{c}{\textbf{HL-Gaussian}}
& \multicolumn{2}{c}{\textbf{quantile loss}} \\
\cmidrule(lr){2-3}\cmidrule(lr){4-5}
\textbf{Backbone}
& \textbf{Gain (\%)}
& \textbf{W / T / L}
& \textbf{Gain (\%)}
& \textbf{W / T / L} \\
\midrule
MLP
& \textbf{+0.65} & 19 / 17 / 5
& +0.55 & 21 / 15 / 5 \\
TabPack
& +0.06 & 8 / 20 / 13
& \textbf{+0.46} & 18 / 21 / 2 \\
RealMLP
& +0.12 & 13 / 17 / 11
& \textbf{+0.40} & 20 / 16 / 5 \\
TabM
& +0.54 & 17 / 19 / 5
& \textbf{+0.58} & 20 / 16 / 5 \\
\bottomrule
\end{tabular}
\end{table}

\section{Quantile regression under heteroskedasticity}
\label{sec:analysis}

The benchmark in Section~\ref{sec:revisiting_exps} identifies several regression objectives that consistently outperform
the MSE baseline. While HL-Gaussian and quantile regression (QR) perform comparably on a plain
MLP, QR is more robust across backbones (Table~\ref{tab:quantile_backbones}). Therefore, we focus on QR for the remainder
of our analysis.

QR predicts conditional quantiles without assuming a parametric family for $p(y \mid x)$, allowing the predicted distribution's shape and scale to vary across inputs. This flexibility is particularly beneficial for tabular data, where the conditional variance $\mathrm{Var}[y \mid x]$, often termed \emph{data uncertainty}, changes substantially with $x$ (i.e., the data is \emph{heteroskedastic}). \citet{kartashev2025unveiling} identify regions of low and high data uncertainty in several established tabular datasets, and demonstrate that dataset-wise aggregate scores often obscure how methods perform across these different regions.

However, modeling input-dependent distributions does not guarantee that regions with different noise levels achieve their optimal predictive performance at the same epoch. Standard global early stopping selects a single checkpoint based on average validation error, applying it uniformly to all inputs. We first examine how prediction errors evolve during training across different uncertainty levels (Section~\ref{subsec:epochs}). We then test whether selecting checkpoints separately for different uncertainty cohorts improves upon doing so for random cohorts of the same size (Section~\ref{subsec:cohort_es}). Together, these experiments reveal an optimization challenge that distributional flexibility alone cannot resolve, motivating the adaptive objective introduced in Section~\ref{sec:ada_quantile}.

Both experiments in this section require a per-point estimate of data uncertainty and are most interpretable when this estimate is exact and test targets are noise-free. To demonstrate the claims of this section in a controlled environment that has these traits, we construct a synthetic dataset $y = f(x) + \varepsilon$,
with noise added only to training targets, so test RMSE directly measures
recovery of $f = \mathbb{E}[y \mid x]$. We use two variants sharing exactly the same
$f$ and $x$: \emph{heteroskedastic}, with $\varepsilon \sim \mathcal{N}(0,
\sigma(x)^2)$, and \emph{homoskedastic}, with $\varepsilon \sim
\mathcal{N}(0, \sigma^2)$, serving as a control. Importantly, for this dataset, levels of true data uncertainty $\sigma^2$ are known to us exactly. See
Appendix~\ref{app:synth_data} for details.

To estimate uncertainty on real data we use a CatBoost model \citep{prokhorenkova2018catboost} trained via Gaussian negative log-likelihood on the real datasets, utilizing the predicted $\sigma(x)$ as our uncertainty estimate. To prevent data leakage, this estimator is fit exclusively on the training split and applied to validation and test points during inference.

\subsection{Optimal stopping epochs vary by uncertainty region}
\label{subsec:epochs}

In this subsection we investigate whether models require different stopping epochs for regions with varying data uncertainty. This requires analyzing the test error disaggregated by uncertainty level across epochs.

The \textit{uncertainty plot} \citep{kartashev2025unveiling} provides
this perspective. Figure~\ref{fig:uncplot}(a) shows an MLP trained with
the quantile loss on the heteroskedastic synthetic dataset. Each curve
corresponds to one epoch and displays the point-wise change in squared
error relative to a reference epoch (the one with the best aggregate
validation MSE), plotted against data uncertainty. Before the reference
epoch, the model performs better on high-uncertainty points and worse
on low-uncertainty ones; afterward, this trend reverses. In other words,
the model starts overfitting high-uncertainty regions before it finishes
fitting low-uncertainty ones, and the reference epoch is merely a
compromise between these competing dynamics.

To verify that this pattern stems from varying noise levels rather than
intrinsic properties of $x$ or $f(x)$, Figure~\ref{fig:uncplot}(c) repeats
the experiment on the homoskedastic variant (same $f$, constant noise
scale, with samples sorted as in the heteroskedastic version). The effect
disappears: performance changes across the displayed epochs are largely unrelated to the uncertainty ordering inherited from the heteroskedastic variant.

Ideally, a model would retain its early-epoch error on high-uncertainty
points while achieving its late-epoch error on low-uncertainty points,
improving upon the QR baseline. However, Figure~\ref{fig:uncplot}(a) only illustrates this misalignment
on synthetic data and does not quantify the resulting performance loss.
Section~\ref{subsec:cohort_es} closes both gaps.

\begin{figure*}[t!]
    \centering
    \includegraphics[width=0.95\textwidth]{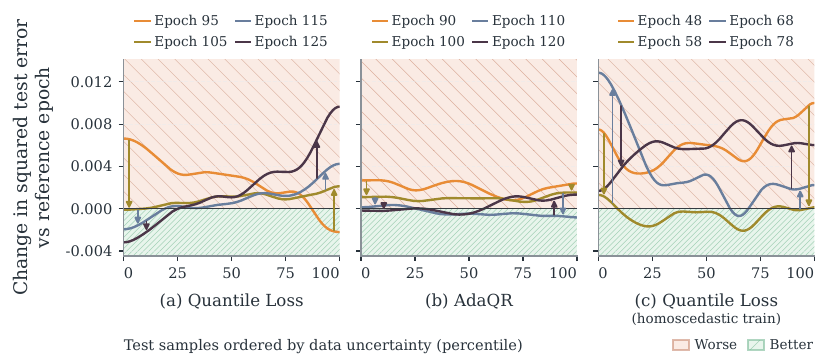}
\caption{\textit{In this synthetic experiment, heteroskedastic noise creates a mismatch between the stopping epochs preferred by low- and high-uncertainty regions, while AdaQR mitigates this effect.}
Uncertainty plots \citep{kartashev2025unveiling} for an MLP trained with (a) quantile loss and (b) AdaQR on the heteroskedastic dataset, and (c) quantile loss on the homoskedastic counterpart.
Each curve shows the smoothed point-wise difference in squared test error relative to the reference epoch with the best aggregate validation MSE; negative values indicate improvement. Test samples are ordered by increasing uncertainty in the heteroskedastic dataset, with the same ordering used in all panels.
Under heteroskedastic noise, earlier epochs perform relatively better on high-uncertainty points and later epochs on low-uncertainty points, indicating that different uncertainty regions favor different stopping times. This separation is reduced with AdaQR and disappears when the noise is homoskedastic.}
    \label{fig:uncplot}
\end{figure*}

\subsection{Measuring the impact of epoch misalignment}
\label{subsec:cohort_es}
To quantify the performance penalty implied by the epoch misalignment in Figure~\ref{fig:uncplot}(a), we evaluate a hypothetical strategy of \emph{per-cohort early stopping}. This approach requires dividing the validation and test sets into groups of comparable data uncertainty, termed \emph{cohorts}. To construct these, we compute the $K - 1$ empirical $\tfrac{k}{K}$-quantiles of the estimated uncertainty $\hat\sigma(x)$ on the validation set. These thresholds are used to split both the validation and test sets into $K$ cohorts, ordered from lowest (cohort $1$) to highest (cohort $K$) uncertainty. Deriving these thresholds strictly from the validation set prevents data leakage, yielding equal-sized validation cohorts and test cohorts that are balanced up to sampling noise. With the cohorts established, we save a model checkpoint at every epoch during training. For each uncertainty cohort $k$, we select the checkpoint $e_k^*$  that minimizes the validation MSE restricted specifically to that group. During inference, each test point is routed to the checkpoint $e_k^*$ corresponding to its uncertainty cohort. Because this approach assigns exactly one checkpoint to each test example, the resulting aggregated test RMSE reflects the empirical benefit of tailoring stopping times to local uncertainty levels.

To accurately isolate this benefit, we cannot simply compare the resulting RMSE to standard global early stopping. While per-cohort selection allows the model to choose $K$ different epochs, each choice relies on only a fraction ($1/K$) of the validation data. To create a fair baseline, we apply the exact same procedure to \textbf{random cohorts} of matching size, created by randomly permuting the points before splitting. These random groups suffer the same penalty of smaller validation sets but contain no uncertainty information. Therefore, comparing uncertainty cohorts directly to random cohorts isolates the advantage of grouping by local noise levels.

Table~\ref{tab:cohorts_real} presents the relative improvements in test RMSE when using uncertainty-based (U) and random (R) cohorts compared to standard global (G) early stopping. To minimize the negative effects of smaller validation sets, we restrict this experiment to the seven largest datasets from our benchmark, reporting results for $K = 5$ and $K = 10$. The results demonstrate that uncertainty-based selection improves median RMSE over global selection by 0.18\% and 0.114\% for five and ten cohorts, respectively. Under our comparison rule, both configurations record three wins, four ties, and no losses across the seven datasets. Furthermore, the drop in overall performance when increasing from $K = 5$ to $K = 10$ cohorts is consistent with increased checkpoint-selection noise from smaller validation subsets.

The table also includes the homoskedastic variant of our synthetic dataset, which shares the same features $x$ and noiseless target function $f(x)$, but features a constant noise scale $\sigma$. This allows us to apply the exact same cohort splits used for the heteroskedastic variant. If $\sigma(x)$ were merely acting as a proxy for partitioning the feature space into regions of similar inputs, both variants would show similar gains. Instead, performance on the homoskedastic version worsens when applying uncertainty-based splits, while the opposite is true for the heteroskedastic variant. Thus, the benefit stems specifically from adapting to how varying label noise alters training dynamics, not from the spatial structure of the inputs.

Per-cohort early stopping, however, is not a practical solution: it requires storing numerous checkpoints and assigning test points to cohorts during inference. These results motivate adapting the objective to delay overfitting in high-uncertainty regions while allowing learning to continue in low-uncertainty regions. Section~\ref{sec:ada_quantile} introduces AdaQR as a potential way to reduce this mismatch during training.

\begin{table}[t]
\centering
\caption{Cohort-based vs. global checkpoint selection for an MLP trained with quantile loss. Improv. (\%) is the relative reduction in mean test RMSE using uncertainty (U/G) or random (R/G) cohorts compared with global selection. All checkpoints are selected on validation data. Bottom rows summarize available comparisons among the seven largest benchmark datasets, excluding synthetic data; W/T/L follows Appendix~\ref{app:comparison}. Results are computed from mean test RMSEs over 15 random seeds.}
\label{tab:cohorts_real}
\begingroup
\setlength{\tabcolsep}{4pt}
\begin{tabular}{@{}lcc@{\hspace{4mm}}cc@{}}
\toprule
& \multicolumn{2}{c}{\textbf{$K=5$}}
& \multicolumn{2}{c}{\textbf{$K=10$}} \\
\cmidrule(lr){2-3}\cmidrule(lr){4-5}
\textbf{Dataset / summary}
& \textbf{U/G} & \textbf{R/G}
& \textbf{U/G} & \textbf{R/G} \\
\midrule
Synthetic (heteroskedastic)
& $+0.318$ & $-0.041$ & $+0.286$ & $-0.118$ \\
Synthetic (homoskedastic)
& $-0.019$ & $-0.021$ & $-0.048$ & $-0.078$ \\
\midrule
\multicolumn{5}{@{}l}{\textit{Seven largest benchmark datasets}} \\
\textbf{Median Improv. (\%)}
& $+0.180$ & $-0.029$ & $+0.114$ & $-0.003$ \\
\textbf{W / T / L}
& 3/4/0 & 0/7/0 & 3/4/0 & 0/7/0 \\
\bottomrule
\end{tabular}
\endgroup
\end{table}

\section{Adaptive quantile regression}
\label{sec:ada_quantile}

To address the stopping-time mismatch identified in Section~\ref{sec:analysis}, we introduce AdaQR, a quantile objective with jointly learned, input-dependent smoothing inspired by Adaptive Label Smoothing \citep{krothapalli2020adaptive}. By applying stronger smoothing in high-uncertainty regions, we aim to postpone overfitting there.

Up to quantile discretization, quantile loss is proportional to the
squared Energy Distance between the predicted distribution and a
Dirac $\delta$ distribution at the observed target $y$.
Appendix~\ref{sec:quantile_crps} details this connection through the
Continuous Ranked Probability Score (CRPS).

To incorporate adaptive smoothing, we replace this point mass with
a Gaussian distribution $\mathcal{N}(y,\sigma_l^2)$ centered around
the target, where $\sigma_l^2$ is a learned variance predicted from
the input $x$.

We design our loss function such that a larger conditional variance
$\mathrm{Var}[y \mid x]$  yields a larger predicted $\sigma_l^2$. Specifically, the loss function we use is
$\mathcal{L}(F,y,\sigma_l)
=
\frac{1}{2}\mathcal{D}^2
\bigl(F,\mathcal{N}(y,\sigma_l^2)\bigr)
+\lambda\sigma_l,$ where $\mathcal{D}$ denotes the Energy Distance
\citep{szekely2003statistics}.
The first term is half of the squared Energy Distance,
where the predicted distribution $F$ is estimated through
the predicted quantiles $q_i$.
The second term acts as a regularization penalty, scaling
the predicted smoothing parameter $\sigma_l$ by a tunable
hyperparameter $\lambda$.
As shown in Appendix~\ref{subapp:adaqr_formula},
when expressed in terms of specific predicted quantiles $q_i$,
our loss becomes:
\begin{equation}
\frac{1}{N}\sum_{i=1}^N
\left(
(q_i-y)
\left(2\Phi\!\left(\tfrac{q_i-y}{\sigma_l}\right)-1\right)
+
\frac{2\sigma_l}{\sqrt{2\pi}}
e^{-\frac{(q_i-y)^2}{2\sigma_l^2}}
\right)
-
\frac{1}{N}\sum_{i=1}^N (2\tau_i-1)q_i
-
\frac{\sigma_l}{\sqrt{\pi}}
+\lambda\sigma_l,
\label{eq:loss}
\end{equation}
where $N$ is the number of predicted quantiles,
$\Phi$ is the CDF of the standard normal distribution,
and $\tau_i$ represents the probability level of the $i$-th quantile.
Assuming Gaussian label noise, Gaussian model predictions, sufficient quantile resolution, and convergence to the Bayes optimum,
Proposition~\ref{prop:adaqr} shows that the optimal smoothing
variance satisfies $\sigma_l^2 \propto \sigma_t^2$,
where $\sigma_t^2=\mathrm{Var}[y \mid x]$.
The full statement and proof are provided in
Appendix~\ref{subapp:adaqr_proofitself}.

\begin{table}[t]
\centering
\setlength{\tabcolsep}{2.5pt}
\caption{Predictive performance on 28 BeyondArena datasets with at least
1,000 training samples.
Gain is the median relative RMSE improvement over MSE on the same
backbone; positive values indicate improvement.
W/T/L summarizes comparisons with MSE using the rule in
Appendix~\ref{app:comparison}.
Bold denotes the largest median gain within each backbone.
AdaQR$^\dagger$ includes auxiliary scale regularization
(Appendix~\ref{app:anchor}).}
\label{tab:adaqr_results}
\begin{minipage}[t]{0.39\linewidth}
\centering
\begin{tabular}{lrc}
\toprule
\multicolumn{3}{c}{\textbf{MLP}} \\
\cmidrule(lr){1-3}
\textbf{Objective} & \textbf{Gain (\%)} & \textbf{W / T / L} \\
\midrule
NLL           & +0.02          & 9 / 12 / 7 \\
$\beta$-NLL   & +0.25          & 10 / 11 / 7 \\
quantile loss & +0.47          & 16 / 10 / 2 \\
AdaQR         & \textbf{+0.66} & 16 / 10 / 2 \\
\bottomrule
\end{tabular}
\end{minipage}
\hfill
\begin{minipage}[t]{0.59\linewidth}
\centering
\begin{tabular}{llcc}
\toprule
\textbf{Backbone} & \textbf{Objective}
& \textbf{Gain (\%)} & \textbf{W / T / L} \\
\midrule
TabPack & quantile loss   & +0.22          & 12 / 14 / 2 \\
        & AdaQR          & \textbf{+0.25} & 14 / 14 / 0 \\
\midrule
RealMLP & quantile loss   & +0.35          & 15 / 11 / 2 \\
        & AdaQR$^\dagger$ & \textbf{+0.82} & 19 / 6 / 3 \\
\midrule
TabM    & quantile loss   & +0.39          & 12 / 12 / 4 \\
        & AdaQR          & \textbf{+0.97} & 14 / 11 / 3 \\
\bottomrule
\end{tabular}
\end{minipage}
\end{table}

These assumptions provide a theoretical framework for AdaQR's
uncertainty adaptation but are not strict prerequisites for
empirical success. In practice, AdaQR adds a scalar output
that parameterizes $\sigma_l$ alongside the predicted quantiles.
This output is used only during training; at inference,
the point prediction is obtained by averaging the quantiles,
as in standard quantile regression. Thus, AdaQR introduces
no additional test-time cost. Implementation details are
provided in Appendix~\ref{app:adaqr_implementation}.

\subsection{AdaQR results}

We first evaluate AdaQR's predictive accuracy. In the following
sections, we analyze the role of adaptive smoothing, the quality
of the learned uncertainty estimates, and AdaQR's effect on the
global stopping compromise. Our main comparison covers 28 datasets with at least 1,000 training
samples, excluding the Tiny category of BeyondArena
\citep{purucker2026beyond}. Since AdaQR trains an additional
$\sigma_l$ logit jointly with the main predictions and requires
tuning $\lambda$, which can complicate optimization on very small
datasets, we exclude such datasets from the main comparison.

We follow our standard evaluation protocol with Muon-based training
on MLP, TabPack \citep{gorishniy2026tabpack},
RealMLP \citep{holzmuller2024realmlp}, and TabM-mini with PLE-Q
embeddings \citep{gorishniy2025tabm}.
On MLP, we additionally evaluate Gaussian NLL and $\beta$-NLL
as heteroskedastic baselines.
For RealMLP, we use a scale-regularized variant, denoted by
AdaQR$^\dagger$, motivated by the uncertainty-estimation analysis and described in
Appendix~\ref{app:anchor}.

\label{sec:adaqr_results}

Table~\ref{tab:adaqr_results} shows that AdaQR achieves positive
median RMSE improvements over MSE on all four backbones.
Its median gains over MSE are also larger than those of quantile
loss, with the clearest differences on TabM and RealMLP
and a small difference on TabPack.
These results support AdaQR as a useful training alternative,
while showing that its benefits vary across backbones.

\subsection{Importance of Adaptive Smoothing}

\begin{table}[h]
\centering
\caption{Ablation of adaptive scale prediction on RealMLP. Gain is the median relative RMSE improvement over MSE. 
W/T/L follows the comparison rule in Appendix~\ref{app:comparison}.}
\label{tab:adaptive_scale_ablation}
\vspace{6pt}
\begin{tabular}{lcc}
\toprule
\textbf{Objective} & \textbf{Gain (\%)} & \textbf{W / T / L} \\
\midrule
quantile loss & +0.35 & 15 / 11 / 2\\
AdaQR$^{\dagger}$ & +0.82 & 19 / 6 / 3\\
Global $\sigma_l$ & +0.14 & 11 / 10 / 7\\
\bottomrule
\end{tabular}
\end{table}

To highlight the importance of adaptivity in AdaQR, we conduct an ablation experiment in which the per-sample $\sigma_l(x)$ is replaced with a single global $\sigma_l$ shared across all instances. Rather than being learned jointly
with the model, the global $\sigma_l$ is fixed and tuned as a hyperparameter. We perform this ablation on RealMLP, a strong-performing tabular
model. We exclude the Tiny subset as described in Section~\ref{sec:adaqr_results}. All other settings follow the
standard experimental protocol described in
Section~\ref{sec:revisiting_exps}.

The substantial
degradation with a global $\sigma_l$ highlights the importance of
instance-wise adaptivity in AdaQR.

\subsection{Uncertainty Estimation}
\begin{figure*}[h!]
    \centering
    \includegraphics[width=\linewidth]{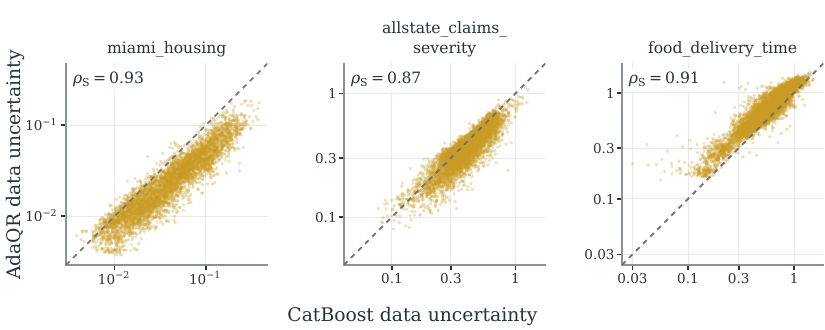}
    \caption{AdaQR versus CatBoost data uncertainty on three
    real-world datasets from the Beyond benchmark (test splits;
    logarithmic axes). AdaQR estimates are averaged across TabM
    branches after conversion to variances using \eqref{eq:sigma_l}.
    Dashed lines indicate $y=x$; $\rho_{\mathrm S}$ denotes
    Spearman's rank correlation computed over all test inputs
    for each dataset. At most 8,000 randomly sampled inputs
    are displayed per dataset.}
    \label{fig:real_uncertainty}
\end{figure*}
We examine whether AdaQR learns meaningful data uncertainty on
  real-world tasks. Since ground-truth uncertainty is unavailable,
  we use CatBoost \citep{prokhorenkova2018catboost} as a trusted reference.
  Figure~\ref{fig:real_uncertainty} shows strong correlations
  between their uncertainty estimates, suggesting AdaQR captures
  meaningful input-dependent variation. See Appendix~\ref{app:real_uncertainty} for additional examples and implementation details.

\subsection{Effect on the Global Stopping Compromise}
\label{sec:adaqr_stopping}

The synthetic results in Figure~\ref{fig:uncplot} suggest
that AdaQR alleviates the mismatch between optimal stopping
epochs across uncertainty regions. To examine whether this
also holds on real-world data, we repeat the cohort-based
experiment from Section~\ref{subsec:cohort_es} for AdaQR on
the same seven largest benchmark datasets, using the MLP
backbone. We compare the benefit of cohort-based over global
checkpoint selection under AdaQR and quantile loss.
\begin{table}[h]
\centering
\caption{Cohort-based vs. global checkpoint selection for an MLP
trained with quantile loss (QR) or AdaQR. Improv. (\%) is the
relative reduction in the mean test RMSE compared with global
selection under the same objective. Results summarize seven
real-world datasets, using 15 random seeds. All checkpoints
are selected on validation data; W/T/L follows
Appendix~\ref{app:comparison}.}
\label{tab:adaqr_cohorts}
\begingroup
\setlength{\tabcolsep}{4pt}
\begin{tabular}{@{}lcc@{\hspace{4mm}}cc@{}}
\toprule
& \multicolumn{2}{c}{\textbf{$K=5$}}
& \multicolumn{2}{c}{\textbf{$K=10$}} \\
\cmidrule(lr){2-3}\cmidrule(lr){4-5}
\textbf{Summary}
& {quantile loss} & {AdaQR}
& {quantile loss} & {AdaQR} \\
\midrule
\textbf{Median Improv. (\%)}
& $+0.180$ & $+0.092$ & $+0.114$ & $+0.085$ \\
\textbf{W / T / L}
& 3/4/ 0 & 1/6/0 & 3/4/0 & 2/5/0 \\
\bottomrule
\end{tabular}
\endgroup
\end{table}

Table~\ref{tab:adaqr_cohorts} shows a lower benefit
from cohort-based selection under AdaQR for both values
of $K$. This benefit decreases on five of the seven datasets
in both settings (see extended results in Appendix~\ref{app:extended_results}),
consistent with our findings on  synthetic data.

\section{Conclusion}

In this paper, we systematically revisit regression objectives for
tabular deep learning, comparing 18 objectives on 41 datasets and
evaluating leading approaches across multiple backbones. Quantile loss
emerges as a consistently strong alternative to carefully tuned MSE.
Yet a strong objective does not eliminate the limitations of global
early stopping: regions with different noise levels can peak at
different training epochs, making a shared checkpoint a compromise.
Motivated by this, we introduce AdaQR, which uses jointly learned,
input-dependent smoothing as one possible remedy and improves
predictions in several settings, with benefits varying across backbones
and datasets. Our findings highlight that progress in tabular regression
requires attention not only to what models optimize, but also to how
different parts of the data are learned over time.

\section{Limitations}

Our study identifies uncertainty-dependent stopping conflicts as an opportunity to improve tabular regression, but our understanding of the underlying training dynamics remains incomplete. AdaQR helps mitigate these stopping conflicts: on real datasets, cohort-based checkpoint selection yields smaller improvements for AdaQR than for standard quantile regression (Section~\ref{sec:adaqr_stopping}), yet still provides additional gains. This residual benefit suggests that adaptive smoothing does not fully resolve differences in preferred stopping times across uncertainty regions. Further work is needed to understand when these conflicts arise and to explore alternative objectives, training procedures, or checkpoint-selection strategies that can address them more effectively. AdaQR’s predictive gains also vary in magnitude across backbones, and its main evaluation excludes Tiny datasets, although results including them are reported in Appendix~\ref{app:extended_results}.

\subsection*{AI use statement}

We used generative AI tools to assist with implementing
experimental code, identifying relevant literature, discussing
research methodology and experimental design, and editing
the manuscript.
The authors reviewed all AI-assisted material and verified
AI-generated code. All final methodological decisions and
scientific conclusions were made by the authors, who take
responsibility for the manuscript, code, and reported results.

\subsection*{Reproducibility statement}

The appendix describes the datasets, preprocessing,
hyperparameter search spaces, and evaluation protocol.
Code and experiment configurations are available in our repository: \url{https://github.com/yandex-research/tabular-regression}.

The repository README provides instructions for downloading
the datasets and reproducing the experiments.

\bibliography{iclr2027_conference}

\begin{thebibliography}{43}
\providecommand{\natexlab}[1]{#1}
\providecommand{\url}[1]{\texttt{#1}}
\expandafter\ifx\csname urlstyle\endcsname\relax
  \providecommand{\doi}[1]{doi: #1}\else
  \providecommand{\doi}{doi: \begingroup \urlstyle{rm}\Url}\fi

\bibitem[Imani and White(2018)]{imani2018improving}
Ehsan Imani and Martha White.
\newblock Improving regression performance with distributional losses.
\newblock In \emph{International conference on machine learning}, pages 2157--2166. PMLR, 2018.

\bibitem[Stewart et~al.(2025)Stewart, Bach, and Berthet]{stewart2025building}
Lawrence Stewart, Francis Bach, and Quentin Berthet.
\newblock Building bridges between regression, clustering, and classification.
\newblock \emph{arXiv preprint arXiv:2502.02996}, 2025.

\bibitem[Farebrother et~al.(2024)Farebrother, Orbay, Vuong, Taiga, Chebotar, Xiao, Irpan, Levine, Castro, Faust, Kumar, and Agarwal]{farebrother2024stopregressingtrainingvalue}
Jesse Farebrother, Jordi Orbay, Quan Vuong, Adrien~Ali Taiga, Yevgen Chebotar, Ted Xiao, Alex Irpan, Sergey Levine, Pablo~Samuel Castro, Aleksandra Faust, Aviral Kumar, and Rishabh Agarwal.
\newblock Stop regressing: Training value functions via classification for scalable deep {RL}.
\newblock In \emph{Forty-first International Conference on Machine Learning}, 2024.
\newblock URL \url{https://openreview.net/forum?id=dVpFKfqF3R}.

\bibitem[Song and Bahri(2025)]{song2025decodingbased}
Xingyou Song and Dara Bahri.
\newblock Decoding-based regression.
\newblock \emph{Transactions on Machine Learning Research}, 2025.
\newblock ISSN 2835-8856.
\newblock URL \url{https://openreview.net/forum?id=avUQ8jguxg}.

\bibitem[Zhu et~al.(2025)Zhu, Yang, and Jerby]{zhu2025gradientalignedregressionpairwise}
Dixian Zhu, Tianbao Yang, and Livnat Jerby.
\newblock Gradient aligned regression via pairwise losses.
\newblock In \emph{Forty-second International Conference on Machine Learning}, 2025.
\newblock URL \url{https://openreview.net/forum?id=1bnH3Q3zL6}.

\bibitem[Zhang et~al.(2023)Zhang, Yang, Mi, Zheng, and Yao]{zhang2023improvingdeepregressionordinal}
Shihao Zhang, Linlin Yang, Michael~Bi Mi, Xiaoxu Zheng, and Angela Yao.
\newblock Improving deep regression with ordinal entropy.
\newblock In \emph{The Eleventh International Conference on Learning Representations}, 2023.
\newblock URL \url{https://openreview.net/forum?id=raU07GpP0P}.

\bibitem[Zha et~al.(2023)Zha, Cao, Son, Yang, and Katabi]{zha2023rank}
Kaiwen Zha, Peng Cao, Jeany Son, Yuzhe Yang, and Dina Katabi.
\newblock Rank-n-contrast: learning continuous representations for regression.
\newblock \emph{Advances in Neural Information Processing Systems}, 36:\penalty0 17882--17903, 2023.

\bibitem[Zhang et~al.(2025)Zhang, Yan, and Yao]{zhang2025improvingdeepregressiontightness}
Shihao Zhang, Yuguang Yan, and Angela Yao.
\newblock Improving deep regression with tightness.
\newblock In \emph{The Thirteenth International Conference on Learning Representations}, 2025.
\newblock URL \url{https://openreview.net/forum?id=dkoiAGjZV9}.

\bibitem[Koenker and Bassett(1978)]{Koenker.Bassett1978}
Roger Koenker and Gilbert Bassett, Jr.
\newblock Regression quantiles.
\newblock \emph{Econometrica}, 46\penalty0 (1):\penalty0 33--50, 1978.
\newblock ISSN 00129682.
\newblock \doi{10.2307/1913643}.
\newblock URL \url{http://www.jstor.org/stable/1913643}.

\bibitem[Qu et~al.(2026)Qu, Holzm{\"u}ller, Varoquaux, and Le~Morvan]{tabiclv2}
Jingang Qu, David Holzm{\"u}ller, Ga{\"e}l Varoquaux, and Marine Le~Morvan.
\newblock Tabiclv2: A better, faster, scalable, and open tabular foundation model.
\newblock In \emph{Forty-third International Conference on Machine Learning}, 2026.

\bibitem[Kartashev et~al.(2025)Kartashev, Rubachev, and Babenko]{kartashev2025unveiling}
Nikolay Kartashev, Ivan Rubachev, and Artem Babenko.
\newblock Unveiling the role of data uncertainty in tabular deep learning.
\newblock \emph{arXiv preprint arXiv:2509.04430}, 2025.

\bibitem[Jordan et~al.(2024)Jordan, Jin, Boza, Jiacheng, Cesista, Newhouse, and Bernstein]{jordan2024muon}
Keller Jordan, Yuchen Jin, Vlado Boza, You Jiacheng, Franz Cesista, Laker Newhouse, and Jeremy Bernstein.
\newblock Muon: An optimizer for hidden layers in neural networks, 2024.
\newblock URL \url{https://kellerjordan.github.io/posts/muon/}.

\bibitem[Gorishniy et~al.(2026{\natexlab{a}})Gorishniy, Rubachev, Feoktistov, and Babenko]{gorishniy2026benchmarking}
Yury Gorishniy, Ivan Rubachev, Dmitrii Feoktistov, and Artem Babenko.
\newblock Benchmarking optimizers for mlps in tabular deep learning.
\newblock \emph{arXiv preprint arXiv:2604.15297}, 2026{\natexlab{a}}.

\bibitem[Huber(1964)]{huber1964robust}
Peter~J Huber.
\newblock Robust estimation of a location parameter.
\newblock \emph{The Annals of Mathematical Statistics}, 35\penalty0 (1):\penalty0 73--101, 1964.

\bibitem[Bishop(1994)]{bishop1994mixture}
Christopher~M Bishop.
\newblock Mixture density networks.
\newblock Technical Report NCRG/94/004, Neural Computing Research Group, Aston University, 1994.

\bibitem[Han et~al.(2022)Han, Zheng, and Zhou]{han2022card}
Xizewen Han, Huangjie Zheng, and Mingyuan Zhou.
\newblock Card: Classification and regression diffusion models.
\newblock \emph{Advances in Neural Information Processing Systems}, 35:\penalty0 18100--18115, 2022.

\bibitem[Heckel and Yilmaz(2020)]{heckel2020early}
Reinhard Heckel and Fatih~Furkan Yilmaz.
\newblock Early stopping in deep networks: Double descent and how to eliminate it.
\newblock \emph{arXiv preprint arXiv:2007.10099}, 2020.

\bibitem[Bao et~al.(2023)Bao, Teng, and Zhang]{bao2023token}
Guangsheng Bao, Zhiyang Teng, and Yue Zhang.
\newblock Token-level fitting issues of seq2seq models.
\newblock In \emph{Proceedings of the 8th Workshop on Representation Learning for NLP (RepL4NLP 2023)}, pages 276--288, 2023.

\bibitem[Berta et~al.(2025)Berta, Holzm{\"u}ller, Jordan, and Bach]{berta2025rethinking}
Eug{\`e}ne Berta, David Holzm{\"u}ller, Michael~I Jordan, and Francis Bach.
\newblock Rethinking early stopping: Refine, then calibrate.
\newblock \emph{arXiv preprint arXiv:2501.19195}, 2025.

\bibitem[Kendall and Gal(2017)]{kendall2017uncertainties}
Alex Kendall and Yarin Gal.
\newblock What uncertainties do we need in bayesian deep learning for computer vision?
\newblock \emph{Advances in neural information processing systems}, 30, 2017.

\bibitem[Immer et~al.(2023)Immer, Palumbo, Marx, and Vogt]{immer2023effective}
Alexander Immer, Emanuele Palumbo, Alexander Marx, and Julia Vogt.
\newblock Effective bayesian heteroscedastic regression with deep neural networks.
\newblock \emph{Advances in Neural Information Processing Systems}, 36:\penalty0 53996--54019, 2023.

\bibitem[Seitzer et~al.(2022)Seitzer, Tavakoli, Antic, and Martius]{seitzer2022on}
Maximilian Seitzer, Arash Tavakoli, Dimitrije Antic, and Georg Martius.
\newblock On the pitfalls of heteroscedastic uncertainty estimation with probabilistic neural networks.
\newblock In \emph{International Conference on Learning Representations}, 2022.
\newblock URL \url{https://openreview.net/forum?id=aPOpXlnV1T}.

\bibitem[Wong-Toi et~al.(2023)Wong-Toi, Boyd, Fortuin, and Mandt]{wong2023understanding}
Eliot Wong-Toi, Alex Boyd, Vincent Fortuin, and Stephan Mandt.
\newblock Understanding pathologies of deep heteroskedastic regression.
\newblock \emph{arXiv preprint arXiv:2306.16717}, 2023.

\bibitem[Stirn et~al.(2023)Stirn, Wessels, Schertzer, Pereira, Sanjana, and Knowles]{stirn2023faithful}
Andrew Stirn, Harm Wessels, Megan Schertzer, Laura Pereira, Neville Sanjana, and David Knowles.
\newblock Faithful heteroscedastic regression with neural networks.
\newblock In \emph{International Conference on Artificial Intelligence and Statistics}, pages 5593--5613. PMLR, 2023.

\bibitem[Tan et~al.(2022)Tan, Wang, and Zhou]{tan2022high}
Kean~Ming Tan, Lan Wang, and Wen-Xin Zhou.
\newblock High-dimensional quantile regression: Convolution smoothing and concave regularization.
\newblock \emph{Journal of the Royal Statistical Society Series B: Statistical Methodology}, 84\penalty0 (1):\penalty0 205--233, 2022.

\bibitem[Fasiolo et~al.(2021)Fasiolo, Wood, Zaffran, Nedellec, and Goude]{fasiolo2021fast}
Matteo Fasiolo, Simon~N Wood, Margaux Zaffran, Rapha{\"e}l Nedellec, and Yannig Goude.
\newblock Fast calibrated additive quantile regression.
\newblock \emph{Journal of the American Statistical Association}, 116\penalty0 (535):\penalty0 1402--1412, 2021.

\bibitem[Jung et~al.(2015)Jung, Lee, and MacEachern]{jung2015efficient}
Yoonsuh Jung, Yoonkyung Lee, and Steven~N MacEachern.
\newblock Efficient quantile regression for heteroscedastic models.
\newblock \emph{Journal of Statistical Computation and Simulation}, 85\penalty0 (13):\penalty0 2548--2568, 2015.

\bibitem[Malekzadeh et~al.(2024)Malekzadeh, Plataniotis, Poulos, and Wang]{malekzadeh2024robust}
Parvin Malekzadeh, Konstantinos~N Plataniotis, Zissis Poulos, and Zeyu Wang.
\newblock A robust quantile huber loss with interpretable parameter adjustment in distributional reinforcement learning.
\newblock In \emph{ICASSP 2024-2024 IEEE International Conference on Acoustics, Speech and Signal Processing (ICASSP)}, pages 6120--6124. IEEE, 2024.

\bibitem[Huber(1992)]{huber1992robust}
Peter~J Huber.
\newblock Robust estimation of a location parameter.
\newblock In \emph{Breakthroughs in statistics: Methodology and distribution}, pages 492--518. Springer, 1992.

\bibitem[Purucker et~al.(2026)Purucker, Tschalzev, Erickson, Blayer, Holzm{\"u}ller, Arazi, Pfefferle, Tajjar, Varoquaux, Hutter, et~al.]{purucker2026beyond}
Lennart Purucker, Andrej Tschalzev, Nick Erickson, Gioia Blayer, David Holzm{\"u}ller, Alan Arazi, Alexander Pfefferle, Mustafa Tajjar, Ga{\"e}l Varoquaux, Frank Hutter, et~al.
\newblock Beyond iid: How general are tabular foundation models, really?
\newblock \emph{arXiv preprint arXiv:2606.30410}, 2026.

\bibitem[Gorishniy et~al.(2025)Gorishniy, Kotelnikov, and Babenko]{gorishniy2025tabm}
Yury Gorishniy, Akim Kotelnikov, and Artem Babenko.
\newblock {TabM}: Advancing tabular deep learning with parameter-efficient ensembling.
\newblock In \emph{The Thirteenth International Conference on Learning Representations}, 2025.
\newblock URL \url{https://openreview.net/forum?id=Sd4wYYOhmY}.

\bibitem[Loshchilov and Hutter(2019)]{loshchilov2018decoupled}
Ilya Loshchilov and Frank Hutter.
\newblock Decoupled weight decay regularization.
\newblock In \emph{International Conference on Learning Representations}, 2019.
\newblock URL \url{https://openreview.net/forum?id=Bkg6RiCqY7}.

\bibitem[Gorishniy et~al.(2026{\natexlab{b}})Gorishniy, Kotelnikov, Rubachev, and Babenko]{gorishniy2026tabpack}
Yury Gorishniy, Akim Kotelnikov, Ivan Rubachev, and Artem Babenko.
\newblock Tabpack: Efficient hyperparameter ensembles for tabular deep learning.
\newblock In \emph{Forty-third International Conference on Machine Learning}, 2026{\natexlab{b}}.
\newblock URL \url{https://openreview.net/forum?id=8qHjrpxrQn}.

\bibitem[Holzm{\"u}ller et~al.(2024)Holzm{\"u}ller, Grinsztajn, and Steinwart]{holzmuller2024realmlp}
David Holzm{\"u}ller, L{\'e}o Grinsztajn, and Ingo Steinwart.
\newblock Better by default: Strong pre-tuned {MLPs} and boosted trees on tabular data.
\newblock \emph{Advances in Neural Information Processing Systems}, 37:\penalty0 26577--26658, 2024.

\bibitem[Prokhorenkova et~al.(2018)Prokhorenkova, Gusev, Vorobev, Dorogush, and Gulin]{prokhorenkova2018catboost}
Liudmila Prokhorenkova, Gleb Gusev, Aleksandr Vorobev, Anna~Veronika Dorogush, and Andrey Gulin.
\newblock Catboost: unbiased boosting with categorical features.
\newblock \emph{Advances in neural information processing systems}, 31, 2018.

\bibitem[Krothapalli and Abbott(2020)]{krothapalli2020adaptive}
Ujwal Krothapalli and A~Lynn Abbott.
\newblock Adaptive label smoothing.
\newblock \emph{arXiv preprint arXiv:2009.06432}, 2020.

\bibitem[Sz{\'e}kely(2003)]{szekely2003statistics}
G{\'a}bor~J Sz{\'e}kely.
\newblock E-statistics: The energy of statistical samples.
\newblock \emph{Bowling Green State University, Department of Mathematics and Statistics Technical Report}, 3\penalty0 (05):\penalty0 1--18, 2003.

\bibitem[Pedregosa et~al.(2018)Pedregosa, Varoquaux, Gramfort, Michel, Thirion, Grisel, Blondel, Müller, Nothman, Louppe, Prettenhofer, Weiss, Dubourg, Vanderplas, Passos, Cournapeau, Brucher, Perrot, and Édouard Duchesnay]{pedregosa2018scikitlearnmachinelearningpython}
Fabian Pedregosa, Gaël Varoquaux, Alexandre Gramfort, Vincent Michel, Bertrand Thirion, Olivier Grisel, Mathieu Blondel, Andreas Müller, Joel Nothman, Gilles Louppe, Peter Prettenhofer, Ron Weiss, Vincent Dubourg, Jake Vanderplas, Alexandre Passos, David Cournapeau, Matthieu Brucher, Matthieu Perrot, and Édouard Duchesnay.
\newblock Scikit-learn: Machine learning in python, 2018.
\newblock URL \url{https://arxiv.org/abs/1201.0490}.

\bibitem[Akiba et~al.(2019)Akiba, Sano, Yanase, Ohta, and Koyama]{akiba2019optuna}
Takuya Akiba, Shotaro Sano, Toshihiko Yanase, Takeru Ohta, and Masanori Koyama.
\newblock Optuna: A next-generation hyperparameter optimization framework.
\newblock In \emph{Proceedings of the 25th ACM SIGKDD international conference on knowledge discovery \& data mining}, pages 2623--2631, 2019.

\bibitem[Gorishniy et~al.(2022)Gorishniy, Rubachev, and Babenko]{gorishniy2022embeddings}
Yury Gorishniy, Ivan Rubachev, and Artem Babenko.
\newblock On embeddings for numerical features in tabular deep learning.
\newblock \emph{Advances in Neural Information Processing Systems}, 35:\penalty0 24991--25004, 2022.

\bibitem[Van Den~Oord et~al.(2017)Van Den~Oord, Vinyals, et~al.]{oord2018neuraldiscreterepresentationlearning}
Aaron Van Den~Oord, Oriol Vinyals, et~al.
\newblock Neural discrete representation learning.
\newblock \emph{Advances in neural information processing systems}, 30, 2017.

\bibitem[Matheson and Winkler(1976)]{matheson1976scoring}
James~E Matheson and Robert~L Winkler.
\newblock Scoring rules for continuous probability distributions.
\newblock \emph{Management science}, 22\penalty0 (10):\penalty0 1087--1096, 1976.

\bibitem[Laio and Tamea(2007)]{laio2007verification}
Francesco Laio and Stefania Tamea.
\newblock Verification tools for probabilistic forecasts of continuous hydrological variables.
\newblock \emph{Hydrology and Earth System Sciences}, 11\penalty0 (4):\penalty0 1267--1277, 2007.

\end{thebibliography}
\bibliographystyle{unsrtnat}

\clearpage
\appendix

\section{Datasets}

\label{app:datasets}

Our benchmark comprises 41 datasets from BeyondArena \citep{purucker2026beyond} . We exclude \texttt{maps\_router\_eta\_1m}, \texttt{delivery\_eta\_1m} and \texttt{cooking\_time\_1m} from the evaluation due to their combination of high numerical dimensionality (992 for \texttt{maps\_router\_eta\_1m}), high categorical cardinality (135131 for \texttt{cooking\_time\_1m}) and large sample size (1.25 million for all three of them). Under our chosen preprocessing and training pipeline, these characteristics make exhaustive hyperparameter tuning infeasible within the 24-hour computational budget imposed for each dataset. We use the lite subset of the benchmark (i.e., the first split of each dataset).

Importantly, GAR[MSE], GAR[MAE] and Ordinal Entropy on MLP  show abnormal behavior and fail to converge on \texttt{airfoil\_self\_noise} and on \texttt{telemonitoring\_parkinsons\_\allowbreak biomedical\_voice\_\allowbreak measurements}. To retain these datasets in the aggregate comparisons, we replace
each failed result with the worst result among the methods
that converged on the same dataset using the MLP backbone.

\paragraph{Batch sizes.}
For non-RealMLP models we determine the batch size from the number of training
examples $n_{\mathrm{train}}$:
\[
B(n_{\mathrm{train}})=
\begin{cases}
32,   & n_{\mathrm{train}} \leq 2{,}800,\\
64,   & 2{,}800 < n_{\mathrm{train}} \leq 4{,}500,\\
128,  & 4{,}500 < n_{\mathrm{train}} \leq 6{,}400,\\
256,  & 6{,}400 < n_{\mathrm{train}} \leq 32{,}000,\\
512,  & 32{,}000 < n_{\mathrm{train}} \leq 108{,}000,\\
1024, & n_{\mathrm{train}} > 108{,}000.
\end{cases}
\]

For RealMLP we use default batch sizes from PyTabkit - 256 for training and 1024 for prediction.

\section{Implementation Details}
\label{app:implementation_details}

\subsection{Experimental protocol}
\label{app:protocol}

\textbf{Data preprocessing.}
For each dataset, the same preprocessing is used. First, we treat binary features as categorical. Both binary and categorical features are then transformed using ordinal encoding. By default, numerical features  are processed via a slightly modified version of quantile normalization from \texttt{Scikit-learn} \citep{pedregosa2018scikitlearnmachinelearningpython} (see the code). 
Finally, targets are standardized to zero mean and unit variance, though all final metrics are computed on the original scale.  The only exception is the evaluation of decoders \citep{song2025decodingbased}, where we follow the authors' protocol by either mapping targets to $[0, 1]$ (for the normalized decoder) or applying no transformations (for the unnormalized version).

\textbf{Training Process}. We use Muon \citep{jordan2024muon} to optimize all layer weight matrices, while AdamW \citep{loshchilov2018decoupled} is used for bias parameters and the output-layer weight matrix. We do not apply learning rate schedulers. We apply global gradient clipping to 1.0. For each dataset, we use the predefined batch size (see Appendix~\ref{app:datasets}) and continue training until  there are $\texttt{patience}$ consecutive epochs without improvements on the validation set; we set $\texttt{patience} = 16 $.

\textbf{Hyperparameter tuning.} For each experiment, we carefully tune optimizer, objective and backbone hyperparameters using Optuna's TPE sampler \citep{akiba2019optuna}, for 100 trials with 24-hour tune time limitation. We use RMSE computed on the validation split as the optimization objective. We provide full parameter spaces in Appendix~\ref{app:tuning}.

\textbf{Evaluation.} After the best hyperparameters are found, we run the model with those parameters on 15 different random seeds and report the performance on the test part of the dataset. For each dataset, we use mean and standard deviation of test metrics (RMSE) over these random seeds to compare methods as described in Appendix~\ref{app:comparison}.

\subsection{Ranks and metrics}
\label{app:comparison}
As we focus exclusively on regression tasks, we report Root Mean Squared Error (RMSE) as a primary metric. As mentioned in Appendix~\ref{app:protocol}, all models and objectives are evaluated on 15 random seeds. We compare methods using two aggregation approaches.

\textbf{Relative Improvement.} For each dataset, we compute the mean RMSE over 15 seeds and calculate its relative improvement (in \%) over the baseline (typically the standard MSE objective). These values are shown in the right panel of Figure~\ref{fig:vertical-revisiging}. Meanwhile,  in the ``Improv. (\%)'' columns of tables, we further aggregate them and report their median over the benchmark datasets. 

\textbf{Ranks.} For each dataset, we assign ranks based on a conservative order relation. Method $A$ (with mean RMSE $A_{\mu}$ and standard deviation $A_{\sigma}$) is considered significantly better than method $B$ if $A_{\mu} + A_{\sigma} < B_{\mu}$ (We use this approach for finding W/T/L with $A$ being the reference method). Intuitively, this approach groups methods into ``tiers,'' ensuring that negligible gains are not counted as wins on individual datasets.

Specifically, ranks are computed independently for each dataset using a top-down procedure. We first sort the methods in ascending order of their mean RMSE. The best model is assigned rank 1 and serves as the initial reference. We then iterate through the remaining models: all methods that are not significantly worse than the current reference (according to the aforementioned rule) also receive rank 1. The first model that fails this criterion is assigned rank 2 and becomes the new reference for subsequent comparisons. This process continues until all models are ranked.

The left panel of Figure~\ref{fig:vertical-revisiging} illustrates the mean and standard deviation of these ranks across the benchmark. 

\subsection{Support construction for distributional losses}
\label{app:support}

For every binning-based distributional objective we construct the support $[a, b]$ as follows: (i) we compute the range between the 0.005- and 0.995-quantiles of the training labels distribution (we employ quantiles rather than min-max values to mitigate the impact of outliers while capturing the effective support of the distribution); (ii) we expand this range by adding the same interval on both the left and right sides; (iii) we then uniformly partition it into bins. 

The uniform binning is the most common in distributional regression research \citep{imani2018improving, stewart2025building} and performs stably in our experiments.

\subsection{Advanced models from Section~\ref{sec:advanced_models}}
\label{app:advanced_models}

Here, we provide implementation details for the experiments in
Section~\ref{sec:advanced_models}. We use TabM-mini with PLE-Q
embeddings \citep{gorishniy2022embeddings}, following the configuration
and implementation of \citet{gorishniy2025tabm}.

We use TabPack equipped with numerical embeddings from \citet{gorishniy2026tabpack}. For all objectives on TabPack we do not add any new sampling parameters: We fix num\_quantiles and num\_bins as 384 (the same as the width) for all probabilistic objectives. Additionally, for AdaQR we fix \texttt{regulizer} as 0.2.

 We use the RealMLP \citep{holzmuller2024realmlp} implementation from PyTabkit\footnote{\url{https://github.com/dholzmueller/pytabkit}} and add HL-Gaussian, quantile regression, AdaQR to that implementation. We use 100 Optuna TPE sampler tuning iterations as for other models and stop tuning at the 24-hour time limit. We use the ``tabarena-new'' hyperparameter search-space.

\subsection{AdaQR implementation}
\label{app:adaqr_implementation}

The implementation of AdaQR includes the following details:

\begin{enumerate}
    \item Alongside the quantile predictions, the model outputs
    a scalar logit $a_\theta(x)$ for the smoothing scale:
    \[
    \sigma_l(x)=|a_\theta(x)+b_0|+\varepsilon,
    \qquad \varepsilon=10^{-5}.
    \]
    Before training, we compute a fixed offset $b_0$ so that
    $a_{\theta_0}(x)+b_0$ has mean $0.5$ over the training set,
    separately for each ensemble member. All scales are expressed
    in standardized target units.

    \item The parameter \texttt{regulizer} in our implementation
    corresponds to $1-\lambda\sqrt{\pi}$, rather than $\lambda$
    itself in the AdaQR objective
    (Section~\ref{sec:ada_quantile}).
    Its tuning space is provided in Appendix~\ref{app:tuning}.

    \item We use equally spaced quantile levels from $0.001$
    to $0.999$. At inference, we average the predicted quantiles,
    as in standard quantile regression; the smoothing-scale
    output is not used.
\end{enumerate}

 \subsection{Hyperparameter Tuning Spaces}
\label{app:tuning}
In this section, we provide tuning spaces for the evaluated objectives and methods. Tables~\ref{tab:mlp_tuning}--\ref{tab:tabpack_tuning} summarize the optimizer and architecture tuning spaces for different models; objective-specific hyperparameter tuning spaces are provided in Table~\ref{tab:loss_tuning}. Parameter names correspond to their implementation in the source code.

In the case of the RealMLP backbone we use the ``tabarena-new'' hyperparameter search-space for RealMLP-specific parameters. 

\begin{table}[h!]
\caption{The hyperparameter tuning space for MLP. Here, (A) is the default space and (B) is used when evaluating Decoder[U] and Decoder[N] (see Section~\ref{sec:results}).}
\label{tab:mlp_tuning}
\centering
\begin{tabular}{ll}
\toprule
Parameter & Distribution \\
\midrule
\# layers & UniformInt[1, 6] \\
Width (hidden size) & UniformInt[64, 1024] \\
Dropout rate & \{0.0, Uniform[0.0, 0.5]\} \\
Learning rate (Muon) & LogUniform[1e-4, 3e-2] \\
Learning rate (AdamW) & (A) LogUniform[3e-5, 1e-3], \; (B) LogUniform[3e-6, 5e-4]\\
Weight decay (AdamW) & LogUniform[5e-3, 5] \\
\bottomrule
\end{tabular}
\end{table}

\begin{table}[h!]
\caption{The hyperparameter tuning space for TabM.}
\label{tab:tabm-mini_ple_tuning}
\centering
\begin{tabular}{ll}
\toprule
Parameter & Distribution \\
\midrule
\# layers & UniformInt[1, 4]\\
Width (hidden size) & UniformInt[64, 1024] \\
Dropout rate & \{0.0, Uniform[0.0, 0.5]\} \\
Learning rate (Muon) & LogUniform[1e-3, 1e-1] \\
Learning rate (AdamW) & LogUniform[5e-5, 3e-3] \\
Weight decay (AdamW) & LogUniform[1e-3, 1] \\
\midrule
d\_embedding & UniformInt[8, 32] \\
n\_bins & UniformInt[2, 128] \\
\bottomrule
\end{tabular}
\end{table}

\begin{table}[h!]
\caption{The sampling space for TabPack}
\label{tab:tabpack_tuning}
\centering
\begin{tabular}{ll}
\toprule
Parameter & Distribution \\
\midrule

\# layers & UniformInt[1, 3] \\
Width (hidden size) & 384 \\
Dropout rate & \{0.0, Uniform[0.0, 0.5]\} \\
Learning rate (Muon) & LogUniform[1e-3, 1e-1] \\
Learning rate (AdamW) & LogUniform[1e-4, 5e-3] \\
Weight decay (AdamW) & LogUniform[1e-3, 1] \\
d\_embedding & UniformInt[8, 32] \\
init\_scale & LogUniform[1e-2, 1e2] \\
\bottomrule
\end{tabular}
\end{table}

\begin{table}[]
\fontsize{8}{10}\selectfont
\caption{Hyperparameter tuning space for different regression objectives. The notation follows Section~\ref{sec:results}; refer to the source code for any naming ambiguities. For TabPack, the numbers of bins and quantiles are fixed at 384,
and AdaQR's \texttt{regulizer} is fixed at 0.2
(Appendix~\ref{app:advanced_models}).}
\label{tab:loss_tuning}
\centering
\begin{tabular}{lll}
\toprule
Objective & Parameter & Distribution \\
\midrule
MSE & --- & --- \\
\midrule
MAE & --- & --- \\
\midrule
Huber & delta & LogUniform[1e-2, 100] \\
\midrule
NLL & --- & --- \\
\midrule
$\beta$-NLL & beta & Uniform[0, 1] \\
\midrule
MDN & mixture\_size & UniformInt[2, 64]\\
\midrule
Hard Binning & n\_bins & UniformInt[32, 1024]  \\
\midrule
{Soft Binning} & n\_bins & UniformInt[32, 1024]  \\
& soft\_sigma & LogUniform[1e-2, 5] \\
\midrule
{HL-Gaussian} & hl\_num\_bins & UniformInt[32, 1024]; \;\, TabPack: 384 \\
& sigma & LogUniform[1e-2, 5] \\
\midrule
{Joint End-To-End} & alpha & \{0.0, LogUniform[1e-5, 0.5]\} \\
& lambda\_auto & LogUniform[1e-4, 1] \\
& lambda\_kl & LogUniform[1e-4, 1] \\
\midrule
GAR[MAE] & alpha & LogUniform[1e-2, 100] \\
\midrule
GAR[MSE] & alpha & LogUniform[1e-2, 100] \\
\midrule
{Decoder[U]} & n\_blocks & UniformInt[1, 6] \\
& d\_block & UniformInt[64, 512] \\
& dropout & \{0.0, Uniform[0.0, 0.5]\} \\
& n\_heads & Categorical[\{1, 2, 4, 8\}] \\
& temperature & LogUniform[0.01, 0.5] \\
& top\_k & UniformInt[1, 9] \\
& base & 10 \\
& num\_exponent\_digits & 1 \\
& num\_mantissa\_digits & 5 \\
\midrule
{Decoder[N]} & n\_blocks & UniformInt[1, 6] \\
& d\_block & UniformInt[64, 512]  \\
& dropout & \{0.0, Uniform[0.0, 0.5]\} \\
& n\_heads & Categorical[\{1, 2, 4, 8\}] \\
& temperature & LogUniform[0.01, 0.5] \\
& top\_k & UniformInt[1, 9] \\
& base & Categorical[\{2, 4, 8, 10\}] \\
& length & UniformInt[4, 16] \\
\midrule
{Softmax-MSE} & min\_value\_relative & Uniform[-0.1, 0.1] \\
& max\_value\_relative & Uniform[0.9, 1.1] \\
& num\_bins & UniformInt[4, 256] \\
\midrule
{Ordinal Entropy} & lambda\_d & LogUniform[1e-5, 0.1] \\
& ordinal\_entropy\_samples & UniformInt[10, 1000] \\
\midrule
{ROT-REG} & lambda\_ot & LogUniform[1e-3, 100] \\
& opt\_transport\_samples & UniformInt[2, 200]  \\
& gamma & LogUniform[0.1, 100] \\
& y\_dim & UniformInt[2, 32] \\
\midrule
{Quantile Loss} & num\_quantiles & UniformInt[32, 1024]\\
\midrule
{AdaQR} & num\_quantiles & UniformInt[32, 1024]; \;\,TabPack: 384\\
& regulizer & Uniform[1e-3, 1.0]; \quad\; TabPack: 0.2\\
\midrule
{$\text{AdaQR}^{\dagger}$} & num\_quantiles & UniformInt[32, 1024]\\
& regulizer & Uniform[1e-3, 1.0]\\
&scale\_anchor\_weight & LogUniform[1e-3, 2] \\ 
\midrule
{AdaQR w. global $\sigma_l$}  & num\_quantiles & UniformInt[32, 1024]\\
 & sigma & LogUniform[1e-4, 8]\\

\bottomrule
\end{tabular}
\end{table}

\section{Detailed description of regression objectives}
\label{app:revisiting}

\textbf{Notation.} We consider supervised regression tasks on tabular data. Let \(D = \{(x_i, y_i)\}_{i = 1}^{n}\) denote the dataset, where $x_i \in \mathbb{R}^d$ is the feature vector and $y_i \in \mathbb{R}$ is the label for the $i$-th datapoint.
We denote the deep learning model by $f_\theta: \mathbb{R}^d \rightarrow \mathbb{R}^{d'}$, where $d'$ is the output dimension, which is not necessarily equal to 1. Let $l(\cdot,\cdot)$ be the loss function. The loss on a given sample is then defined as $l = l(f_{\theta}(x), y)$.

\subsection{Losses}

\label{app:losses}

\subsubsection{Traditional losses}
The first group of objectives in our comparison includes canonical regression losses. Here, any values denoted by a hat $\hat{\cdot}$ correspond to predictions produced by our model.

\begin{itemize}
\item \textbf{MSE} (Mean Squared Error, or $l_2$ loss) is the standard objective for most regression tasks:
    \[l = l_2(\hat{y}, y) =(\hat{y} - y)^2\]
\item \textbf{MAE} (Mean Absolute Error, or $l_1$ loss) provides robustness against outliers by penalizing errors linearly: \[l = l_1(\hat{y}, y) = |\hat{y} - y|\]
\item \textbf{Huber loss} \citep{huber1992robust} is a middle ground between MSE and MAE:
    \[
    l = l_\delta(\hat{y}, y) = 
    \begin{cases}\frac{1}{2} (y - \hat{y})^2, & \text{if } | \hat{y} - y| \le \delta \\
    \delta \cdot |y - \hat{y}| - \frac{1}{2} \delta^2, & \text{if } |\hat{y} - y| > \delta\end{cases}
    \]
\end{itemize}
\subsubsection{Binning-based losses} \label{sec:binning}

Among various regression objectives, the binning-based losses \citep{oord2018neuraldiscreterepresentationlearning, imani2018improving, stewart2025building} stand out as a promising option for regression tasks and are reported to be beneficial across diverse domains \citep{farebrother2024stopregressingtrainingvalue}. The core idea is to discretize scalar targets using a grid of consecutive bins. Formally, consider a sample $(x, y) \in D$. The regression label $y$ is mapped to a probability vector $\mathbf{p} = (p_1, \dots, p_{d'})$ over the $d'$ bins on the support $[a, b]$. Here, the $i$-th bin is the interval $[l_i, r_i]$ with center $c_i$.  Given the model's logits $f_{\theta}(x)$, let $\mathbf{q}= \text{softmax}(f_{\theta}(x))$. The model is trained to approximate the target distribution $\mathbf{p}$ via the KL divergence:
\[
l = l_{\textsc{KL}} (f_\theta(x), y) = D_{\textsc{KL}}(\mathbf{p}\,\|\,\mathbf{q}).
\]
At inference, model predictions are computed by taking the expectation over the bins:
\(
\hat{y}   = \sum_i \mathbf{q}_i \cdot c_i.
\)

The particular form of the target distribution is typically the same for all datapoints and using different forms of the target distributions results in different regression objectives. Below, we describe several binning-based objectives that we include in the comparison.

\textbf{Hard Binning} \citep{oord2018neuraldiscreterepresentationlearning, stewart2025building} maps each $y$ to a one-hot vector $\mathbf{p}$ by assigning it to the nearest bin center.

\textbf{Soft Binning} \citep{stewart2025building} improves Hard Binning by performing soft mapping to the dense probability vector:
    \[\mathbf{p} = \mathrm{softmax}\left( - \frac{|c_1 - y|^2}{2\sigma^2}, \dots, - \frac{|c_{d'} - y|^2}{2\sigma^2} \right), \]
    where $\sigma$ is a global (shared across the dataset) hyperparameter controlling the spread of probability mass over the bins.
    
 \textbf{HL-Gaussian} \citep{imani2018improving} constructs $\mathbf{p}$ by discretizing the truncated Gaussian distribution with mean $y$ and fixed variance $\sigma^2$ over $[a, b]$. Specifically, given the density
 \[
 \rho(t \mid y) = \frac{1}{Z\sqrt{2\pi \sigma^2}} e^{-\frac{(t-y)^2}{2 \sigma^2}}
 \]
the bin probabilities are
 \[\mathbf{p}_i = \frac{1}{Z} \left(\Phi\left(\frac{r_i - y}{\sigma}\right) - \Phi\left(\frac{l_i - y}{\sigma}\right)\right) \]
where $\Phi$ is the cdf-function for $\mathcal{N}(0,1)$
and  $Z$   is the normalization constant. Similarly to Soft Binning, $\sigma$ is a global hyperparameter that effectively determines the smoothness of $\mathbf{p}$.

\textbf{Joint end-to-end encoding} \citep{stewart2025building} is a binning-based objective where the target distribution $\mathbf{p}$ itself is learned. Instead of using a fixed rule (as in the methods above), the mapping $y \mapsto \mathbf{p}$ is parameterized by an auxiliary autoencoder-style model and trained end-to-end along with the main model. This is achieved by adding a reconstruction term (recovering $y$ from $\mathbf{p}$) to the $D_{\textsc{KL}}(\mathbf{p}\,\|\,\mathbf{q})$ loss.
 
\subsubsection{Quantile loss} \label{sec:quantile_regression}

Quantile loss \citep{Koenker.Bassett1978} offers an alternative approach to modeling conditional distributions $p(y\mid x)$ by directly predicting specific quantiles of the target variable. Specifically, for a given input $x$, the model predicts $N$ quantiles $\hat{y}_{\tau_1}, \hat{y}_{\tau_2}, \dots, \hat{y}_{\tau_N}$ corresponding to quantile levels $\tau_1, \tau_2, \dots, \tau_N \in (0, 1)$.

Each quantile prediction $\hat{y}_{\tau_i}$ is trained using the pinball loss function:
\[
\rho_{\tau_i}(y, \hat{y}_{\tau_i}) = 
\begin{cases}
\tau_i (y - \hat{y}_{\tau_i}), & \text{if } y \geq \hat{y}_{\tau_i} \\
(1 - \tau_i) (\hat{y}_{\tau_i} - y), & \text{if } y < \hat{y}_{\tau_i}
\end{cases}
\]
where $\tau_i \in (0, 1)$ represents the target quantile level. The overall loss function is then computed as the average pinball loss across all $N$ quantiles:
\[
l = l_{\textsc{QR}}(f_\theta(x), y) = \frac{1}{N} \sum_{i=1}^{N} \rho_{\tau_i}(y, f_{\theta,i}(x))
\]
where $f_{\theta,i}(x)$ denotes the model's prediction for the $i$-th quantile.

During inference, the final prediction $\hat{y}$ is typically obtained by taking the average of all predicted quantiles:
\(
\hat{y} = \frac{1}{N} \sum_{i=1}^{N} \hat{y}_{\tau_i}.
\)

This approach provides a more direct characterization of the conditional distribution compared to binning methods, while avoiding the discretization of the output space.

\subsubsection{Modeling full conditional distributions in the arbitrary form}

An alternative paradigm to solve a regression problem has been employed in several recent works \citep{han2022card, song2025decodingbased} that model the conditional distributions $p(y|x)$ without any assumptions about their parametric form. In our experimental comparison, we include the most recent approach proposed by \citet{song2025decodingbased} that
relies on the autoregressive decoder. Specifically, the approach uses a small autoregressive decoding head on top of an MLP backbone. Both MLP and decoder are trained jointly for the next-token (where tokens represent digits) prediction with a cross-entropy loss. 

The decoder predicts the sequence of tokens $t_1, \dots, t_k$ that represent the target $y$ in an autoregressive manner.
Formally, let $\phi(x)$ be the representation produced by the MLP encoder. The model estimates the conditional distribution $p(y\mid x)$ via the following factorization:
\[ p(y \mid x) = \prod_{i = 1}^{k} p(t_i \mid t_1, \dots, t_{i - 1}, \phi(x)). \]

In our study, we evaluate two versions of this decoder described by the method's authors. In the \emph{normalized} variant, targets are mapped to the $[0, 1]$ interval using min-max transformation before tokenization. In the \emph{unnormalized} variant, targets are represented as floating-point numbers in a base-$B$ system: $y = \pm B^e \cdot m$, where the exponent $e$ and mantissa $m \in [1, B)$ are predicted.

\subsubsection{Other losses}
In this section we list several additional regression objectives that are included in our evaluation. 
\begin{itemize}
    \item \textbf{GAR} (Gradient aligned regression) \citep{zhu2025gradientalignedregressionpairwise}  combines point-wise and pair-wise terms to better match the gradient and magnitude of the target function.

    \item \textbf{Softmax-MSE} \citep{stewart2025building} uses a softmax binning head followed by computing the expectation and MSE loss:
\[
l = \left(\textstyle\sum_i \mathbf{q}_i \cdot c_i - y\right)^2
\]

    where $c_i$ and $\mathbf{q}$ follow the binning losses notation from Section~\ref{sec:binning}.
    \item \textbf{Ordinal Entropy loss} \citep{zhang2023improvingdeepregressionordinal} is an objective that enables NNs to learn better feature representations $Z$ (specifically, to maximize mutual information \(\mathcal{I}(Z, Y) = \mathcal{H}(Z) - \mathcal{H}(Z \mid Y)\)), by adding two additional terms to the initial regression loss.
    \item \textbf{ROT-REG} \citep{zhang2025improvingdeepregressiontightness} is an optimal transport based regularizer that similarly allows the model to enhance learned representations by tightening them.

    \item \textbf{NLL} assumes $p(y|x) = \mathcal{N}(\mu(x), \sigma^2(x))$ and minimizes the negative log-likelihood:
    \[
    l = \frac{1}{2}\log \hat{\sigma}^2 + \frac{(\hat{\mu} - y)^2}{2\hat{\sigma}^2}.
    \]
    \item \textbf{$\beta$-NLL} \citep{seitzer2022on} extends \textbf{NLL} by reweighting each sample by $\lfloor\hat{\sigma}^{2\beta}\rfloor$ to improve optimization stability, where $\lfloor \cdot \rfloor$ denotes the stop-gradient operation.

    \item \textbf{Mixture Density Networks} \citep{bishop1994mixture} model the conditional distribution $p(y|x)$ as a Gaussian mixture.
\end{itemize}

\section{Detailed description of the dataset used in Section~\ref{sec:analysis}}
\label{app:synth_data}

To evaluate the robustness of our method under strictly controlled noise conditions, we construct both homoskedastic and heteroskedastic synthetic datasets. Below are the details regarding the process of generation.

\textbf{Dataset Size and Splits:}

The total dataset size is $N = 50,000$. We split this into 20,000 training examples, 15,000 validation examples, and 15,000 test examples. Importantly, while the training targets contain injected noise, the validation and test targets are completely noise-free ($y_{\text{clean}}$). This allows us to accurately evaluate how well the models uncover the true underlying conditional expectation.

\textbf{Base Network Architecture:}

We use a standardized Multi-Layer Perceptron (MLP) generator to create non-linear mappings. To ensure stable scales across different initializations, the outputs of this generator are explicitly standardized to have zero mean and unit variance using statistics calculated over an initial batch of $10,000$ Gaussian noise samples. All MLPs use hidden dimensions of 256 and ReLU activations. 

\textbf{Feature Generation ($X$):}
\begin{enumerate}
    \item We sample a 40-dimensional latent from the standard multivariate Gaussian distribution: $z \sim \mathcal{N}(0, I_{40})$.
    \item We map $z$ through a 3-layer MLP generator to produce a 100-dimensional informative feature vector $x_{\text{real}} \in \mathbb{R}^{100}$. We use this step to imitate the complex relationships between features that are often found in tabular data.
    \item To introduce irrelevant features, we sample 16-dimensional standard Gaussian noise $x_{\text{noisy}} \sim \mathcal{N}(0, I_{16})$.
    \item The final feature vector is the concatenation of the informative and noisy features: $x = [x_{\text{real}}, x_{\text{noisy}}] \in \mathbb{R}^{116}$.
\end{enumerate}

\textbf{Target Generation ($Y$):}
The clean target values, $f(x)$, are generated by passing the informative features $x_{\text{real}}$ through a 2-layer MLP $f_\theta$. 
To create the noisy training targets $y_{\text{noisy}}$, we inject Gaussian noise $\epsilon \sim \mathcal{N}(0, 1)$ according to one of two settings:

\begin{itemize}
    \item \textbf{Homoskedastic Dataset:} The variance is constant across the entire input space. 
    \[
    y_{\text{noisy}} = y_{\text{clean}} + \sigma \epsilon, \quad \text{where } \sigma = 4.0
    \]

    \item \textbf{Heteroskedastic Dataset:} The standard deviation $\sigma(x)$ depends non-linearly on the informative features. We pass $x_{\text{real}}$ through another 2-layer MLP $g_\phi$ to obtain a scalar. This scalar is pushed through a Standard Normal Cumulative Distribution Function $\Phi$ to bound it in $(0, 1)$. We then scale, shift, and exponentiate it to create a smooth, variable standard deviation profile:
    \[
    \sigma(x) = \exp \Big( \Phi(g_\phi(x_{\text{real}})) \times 4.5 - 2.0 \ \Big)
    \]
    \[
    y_{\text{noisy}} = y_{\text{clean}} + \sigma(x) \epsilon
    \]
\end{itemize}

To ensure exact reproducibility, the process uses manual random seeds before any generation.

\section{Quantile loss and Energy Distance}
\label{sec:quantile_crps}

There exists a well-known relation between the quantile loss
and the Continuous Ranked Probability Score (CRPS)
\citep{matheson1976scoring, laio2007verification}:
\[
\begin{aligned}
\mathrm{CRPS}(F,y)
&=
\int_{-\infty}^{+\infty}
\bigl(F(z)-\mathbbm{1}\{z\geq y\}\bigr)^2\,dz \\
&=
2\int_0^1 \rho_\tau(y,Q_F(\tau))\,d\tau
\approx
2\cdot l_{\mathrm{QR}}(Q_F(\cdot),y),
\end{aligned}
\]
where $F$ is a CDF, $Q_F$ is its quantile function,
and $\rho_\tau$ is the pinball loss
(Section~\ref{sec:quantile_regression}).
The approximation replaces the integral over quantile levels
with an average over the finite uniform grid used by our models.
Thus, up to this discretization and a constant factor,
minimizing the quantile loss is equivalent to minimizing CRPS.

For real-valued random variables with finite first moments,
the squared Energy Distance is
\citep{szekely2003statistics}
\[
\mathcal{D}^2(F,G)
=
2\mathbb{E}|X-Y|
-
\mathbb{E}|X-X'|
-
\mathbb{E}|Y-Y'|,
\]
where $X,X'\sim F$ and $Y,Y'\sim G$ are mutually independent.
When $G=\delta_y$, the Dirac distribution concentrated at $y$,
this gives
\[
\mathrm{CRPS}(F,y)
=
\mathbb{E}|X-y|
-\frac{1}{2}\mathbb{E}|X-X'|
=
\frac{1}{2}\mathcal{D}^2(F,\delta_y).
\]
This identity provides the Energy Distance interpretation of
quantile loss used in Section~\ref{sec:ada_quantile}:
the predicted distribution is compared with a point mass
at the observed target.
AdaQR replaces this point mass with an input-dependent
Gaussian smoothing distribution.

\section{AdaQR: Full Objective and Proof of Proposition~\ref{prop:adaqr}}
\label{app:adaqr_proof}

\subsection{Full objective in terms of quantiles}
\label{subapp:adaqr_formula}

Denoting the predicted quantiles as $q_i$ at levels $0 \leq \tau_i \leq 1$,
$i = 1,\dots,N$, the AdaQR loss expands as follows:

\begin{align}
&\frac{1}{2}\mathcal{D}^2\bigl(F,\, \mathcal{N}(y, \sigma_l^2)\bigr) 
+ \lambda \sigma_l \notag\\
= \,&\mathop{\mathbb{E}}\limits_{X, X'\sim F}
\left[\mathop{\mathbb{E}}\limits_{Y, Y'\sim\mathcal{N}(y, \sigma_l^2)}
\left[\|X - Y\| - \dfrac{1}{2} \|X - X'\| - \dfrac{1}{2} \| Y - Y'\| 
\right]\right] + \lambda \sigma_l \notag\\
\approx \,&\mathop{\mathbb{E}}\limits_{Y, Y'\sim\mathcal{N}(y, \sigma_l^2)}
\left[\dfrac{1}{N}\sum\limits_{i=1}^N \|q_i-Y\| 
- \dfrac{2}{N} \sum\limits_{i=1}^{N} \tau_i q_i 
+ \dfrac{1}{N} \sum\limits_{i=1}^{N} q_i 
- \dfrac{1}{2}\mathbb{E}\|Y - Y'\|\right] + \lambda \sigma_l \notag\\
= \,&\frac{1}{N}\sum_{i=1}^N \left( (q_i - y) 
\left(2\Phi\!\left(\tfrac{q_i - y}{\sigma_l}\right) - 1\right) 
+ \frac{2\sigma_l}{\sqrt{2\pi}} e^{-\frac{(q_i-y)^2}{2\sigma_l^2}} \right)
- \frac{2}{N}\sum_{i=1}^N \tau_i q_i 
+ \frac{1}{N}\sum_{i=1}^N q_i 
- \frac{\sigma_l}{\sqrt{\pi}} + \lambda\sigma_l,
\label{eq:adaqr_quantiles}
\end{align}
where $\Phi(\cdot)$ is the standard normal cdf.

\subsection{Proof of Proposition~\ref{prop:adaqr}}
\label{subapp:adaqr_proofitself}

\begin{proposition}[\textbf{AdaQR learns proportional uncertainty}]
\label{prop:adaqr}
Under the following assumptions:
\begin{itemize}
    \item[\textbf{(A1)}\label{ass:A1}] the number of quantiles is large 
    enough to ignore discretization effects;
    \item[\textbf{(A2)}\label{ass:A2}] both the predicted distribution $F$ 
    and the target distribution are Gaussian;
    \item[\textbf{(A3)}\label{ass:A3}] the optimizer converges to 
    Bayes-optimal predictions;
    \item[\textbf{(A4)}\label{ass:A4}] the label noise is Gaussian: 
    $y \mid y_{\mathrm{true}} \sim \mathcal{N}(y_{\mathrm{true}},\, 
    \sigma_t^2)$,
    \item[\textbf{(A5)}\label{ass:A5}] $0<\lambda<1/\sqrt{\pi}$ and $\sigma_t>0$
\end{itemize}
the AdaQR objective~(\ref{eq:loss}) has a unique optimum in 
$\sigma_l$, and the optimal smoothing parameter satisfies:
\begin{equation}
    \sigma_l^2 = \frac{(1 - \sqrt{\pi}\lambda)^2}{1 - 
    (1 - \sqrt{\pi}\lambda)^2}\,\sigma_t^2.
    \label{eq:sigma_l}
\end{equation}
In particular, $\sigma_l^2 \propto \sigma_t^2$, so AdaQR correctly 
captures the relative heteroskedasticity of the data.
\end{proposition}

\begin{proof}
Fix an input $x$ and define the expected objective
\[
\overline{\mathcal{L}}(\sigma_q,\sigma_l)
= \mathbb{E}_{y \mid x}\!\left[\mathcal{L}(F, y, \sigma_l)\right].
\]
By \textbf{(A2)} and \textbf{(A3)}, we model
$F = \mathcal{N}(y_{\mathrm{true}}, \sigma_q^2)$, and by \textbf{(A4)},
for each observed label $y$ the target distribution is
$G_y = \mathcal{N}(y, \sigma_l^2)$ with
$y \sim \mathcal{N}(y_{\mathrm{true}}, \sigma_t^2)$.

The three terms of the Energy Distance simplify via the Half-Normal
expectation formula $\mathbb{E}|Z| = \sqrt{2/\pi}\,\sigma$ for
$Z \sim \mathcal{N}(0, \sigma^2)$. Note that the relevant quantities
are the \emph{centered} differences
$X - y_{\mathrm{true}}$, $Y - y$, and $y - y_{\mathrm{true}}$,
which are independent and zero-mean (by \textbf{(A2)} and
\textbf{(A4)}):
\begin{itemize}
    \item $X - Y = (X - y_{\mathrm{true}}) - (y - y_{\mathrm{true}})
    - (Y - y)$ is $\mathcal{N}(0,\, \sigma_q^2 + \sigma_t^2 +
    \sigma_l^2)$;
    \item $X - X' = (X - y_{\mathrm{true}}) - (X' - y_{\mathrm{true}})$
    is $\mathcal{N}(0,\, 2\sigma_q^2)$, since $X, X'$ are i.i.d.
    samples from $F$;
    \item conditionally on a shared label $y$, the samples
    $Y, Y' \mid y$ are i.i.d.\ from $\mathcal{N}(y, \sigma_l^2)$, so
    $Y - Y' \mid y$ is $\mathcal{N}(0,\, 2\sigma_l^2)$; crucially,
    this conditional law does not depend on $y$, hence it also holds
    marginally after averaging over $y \mid x$.
\end{itemize}

Combining these and taking the expectation over $y \mid x$, the first
term of the expected objective equals:
\[
\mathbb{E}_{y\mid x}\!\left[\tfrac{1}{2}\mathcal{D}^2(F, G_y)\right]
= \frac{1}{\sqrt{\pi}}\Bigl(\sqrt{2(\sigma_q^2 + \sigma_t^2 +
\sigma_l^2)} - \sigma_q - \sigma_l\Bigr),
\]
so, under \textbf{(A1)}--\textbf{(A5)}, the expected objective is:
\[
\overline{\mathcal{L}}(\sigma_q, \sigma_l) =
\frac{1}{\sqrt{\pi}}\Bigl(\sqrt{2(\sigma_q^2 + \sigma_t^2 +
\sigma_l^2)} - \sigma_q - \sigma_l\Bigr) + \lambda\sigma_l.
\]

Setting partial derivatives to zero:
\[
\begin{cases}
\dfrac{\partial \overline{\mathcal{L}}}{\partial \sigma_q} =
\dfrac{1}{\sqrt{\pi}}\!\left(\dfrac{2\sigma_q}{\sqrt{2(\sigma_q^2 +
\sigma_t^2 + \sigma_l^2)}} - 1\right) = 0,\\[8pt]
\dfrac{\partial \overline{\mathcal{L}}}{\partial \sigma_l} =
\dfrac{1}{\sqrt{\pi}}\!\left(\dfrac{2\sigma_l}{\sqrt{2(\sigma_q^2 +
\sigma_t^2 + \sigma_l^2)}} - 1\right) + \lambda = 0.
\end{cases}
\]

The first equation gives $\sigma_q^2 = \sigma_l^2 + \sigma_t^2$.
Substituting into the second:
\[
\frac{4}{\pi}\sigma_l^2 = \left(\frac{1}{\sqrt{\pi}} -
\lambda\right)^2 \cdot 4(\sigma_l^2 + \sigma_t^2)
\implies
\sigma_l^2 = (1 - \sqrt{\pi}\lambda)^2(\sigma_l^2 + \sigma_t^2),
\]
and solving for $\sigma_l^2$ yields~\eqref{eq:sigma_l}.

It remains to verify that this stationary point is the unique minimum.

For $\sigma_t > 0$, the function $\sqrt{2(\sigma_q^2+\sigma_l^2+\sigma_t^2)}$ is strictly convex in $(\sigma_q,\sigma_l)$ because its Hessian is positive definite everywhere.
The remaining terms,
$-\sigma_q/\sqrt{\pi} + (-1/\sqrt{\pi} +\lambda)\sigma_l$, are linear.
Hence $\overline{\mathcal{L}}$ is strictly convex, and under
\textbf{(A5)} the stationary point found above — which is positive —
is its unique global minimum.
\end{proof}

\section{Scale regularization for RealMLP}
\label{app:anchor}
For RealMLP, we augment AdaQR with an auxiliary scale regularizer,
denoting the resulting objective by AdaQR$^{\dagger}$. Below, we
describe its motivation, formulation, and effect on uncertainty
estimation.

\subsection{Uncertainty estimation in RealMLP}
Proposition~\ref{prop:adaqr} characterizes the population optimum
under Gaussian assumptions. In practice, however, there is no
guarantee that training reaches this optimum. In particular,
a model can correctly rank inputs by uncertainty while producing
inaccurate absolute scales.

We compare TabM and RealMLP on our heteroskedastic synthetic
dataset, where the true data uncertainty is known, converting
smoothing scales into noise-scale estimates via~\eqref{eq:sigma_l}.
TabM already recovers the underlying data uncertainty reasonably
well and consistently across ensemble members
(Figure~\ref{fig:anchor_uncertainty}(a)). RealMLP shows larger
scale deviations and greater variation across members
(Figure~\ref{fig:anchor_uncertainty}(b)), suggesting that these
difficulties limit the gains of vanilla AdaQR on this backbone.
This motivates an auxiliary regularizer for the scales inferred
from the quantile spread and the smoothing output.

\subsection{Auxiliary scale regularization}
We augment AdaQR with a regularizer that provides a direct
training signal for the predicted scales. Following the notation
of the proof in Appendix~\ref{app:adaqr_proof}, let
$r=1-\sqrt{\pi}\lambda$, with $0<r<1$. Under the assumptions
of Proposition~\ref{prop:adaqr}, the optimal scales satisfy
\[
\sigma_t
=
\sqrt{1-r^2}\,\sigma_q
=
\frac{\sqrt{1-r^2}}{r}\,\sigma_l.
\]
This motivates two estimates of the noise standard deviation:
\[
s_q=\sqrt{1-r^2}\,\widehat{\sigma}_q,
\qquad
s_\sigma=\frac{\sqrt{1-r^2}}{r}\,\sigma_l,
\]
where, for $N$ predicted quantiles $q_j$ at levels $\tau_j$,
we estimate the quantile-distribution scale as
\[
\widehat{\sigma}_q^2
=
\frac{\frac{1}{N}\sum_{j=1}^{N}(q_j-\bar q)^2}
     {\frac{1}{N}\sum_{j=1}^{N}[\Phi^{-1}(\tau_j)]^2},
\qquad
\bar q=\frac{1}{N}\sum_{j=1}^{N}q_j.
\]
Here, $\Phi$ is the standard normal CDF. The denominator
accounts for the finite, symmetric quantile grid, recovering
$\sigma_q^2$ exactly for Gaussian quantiles.

We supervise both estimates using CRPS evaluated at Gaussian
CDFs
\[
F_{\mu,s}(z)=\Phi\!\left(\frac{z-\mu}{s}\right),
\qquad
\mu=\operatorname{stopgrad}(\bar q),
\]
where $\operatorname{stopgrad}$ preserves its argument's value
but blocks gradients through it. The Gaussian CRPS has
the closed form
\[
\text{CRPS}(F_{\mu,s},y)
=
e\bigl(2\Phi(e/s)-1\bigr)
+2s\phi(e/s)-\frac{s}{\sqrt{\pi}},
\qquad e=y-\mu,
\]
where $\phi$ is the standard normal density.
The resulting objective, denoted by AdaQR$^{\dagger}$, is
\[
\mathcal{L}_{\mathrm{AdaQR}^{\dagger}}
=
\mathcal{L}_{\mathrm{AdaQR}}
+
\frac{\gamma}{2}
\left[
\text{CRPS}(F_{\mu,s_q},y)
+
\text{CRPS}(F_{\mu,s_\sigma},y)
\right],
\]
averaged over the training batch, with $\gamma\geq0$ controlling
the regularization strength.

Detaching $\bar q$ blocks the auxiliary gradient through the
predicted mean, while centering the quantiles makes the spread
estimate invariant to a common shift. The regularizer thus
directly trains the quantile spread and the smoothing scale,
although it can still affect point predictions through the
shared backbone.

Under the assumptions of Proposition~\ref{prop:adaqr},
both inferred scales equal $\sigma_t$ at the population optimum.
With the conditional mean correctly specified, $s=\sigma_t$
also minimizes the expected Gaussian CRPS. The auxiliary terms
therefore share the same idealized scale optimum as AdaQR.

The regularizer requires no architectural changes: we retain
AdaQR's scalar smoothing-scale output. At inference time,
the point prediction remains the average of the predicted
quantiles.

\subsection{Effect on RealMLP}
\label{app:realmlp_effect}
As shown in Figure~\ref{fig:anchor_uncertainty}(b) and (c),
AdaQR$^{\dagger}$ produces noise-scale estimates that more closely
match the true data uncertainty and are more consistent across
ensemble members. The improvement primarily concerns the
absolute scales, as vanilla AdaQR already ranks inputs by
uncertainty reasonably well. These results suggest that the
auxiliary regularizer helps RealMLP recover the scale relationship
predicted by our theoretical analysis.

\begin{figure}
    \centering
    \includegraphics[width=0.9\linewidth]{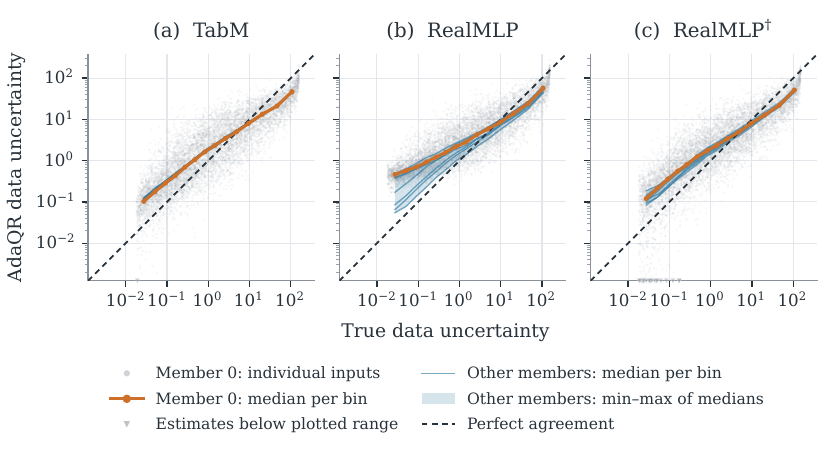}
    \caption{Data uncertainty estimated by AdaQR versus the true data uncertainty
  on a heteroskedastic synthetic dataset, shown on logarithmic axes.
  Both uncertainties are expressed as variances. AdaQR estimates are
  obtained using \eqref{eq:sigma_l}, following the theoretical results
  in Section~\ref{sec:ada_quantile}.
  Panels compare (a) TabM, (b) RealMLP, and (c) RealMLP with
  scale-regularized AdaQR$^{\dagger}$.
  Gray points show individual test inputs for one ensemble
  member; the orange line connects this member's median estimates
  in equal-count bins of true uncertainty.
  Each blue line shows the corresponding bin medians for another
  ensemble member. The shaded band spans the minimum and maximum
  of these medians, rather than a confidence interval.
  Dashed diagonals indicate perfect agreement; downward triangles
  mark estimates below the displayed range.
  All three models capture relative variation in data uncertainty.
  Scale regularization reduces RealMLP's deviations at low uncertainty
  and improves consistency across ensemble members.}
    \label{fig:anchor_uncertainty}
\end{figure}

\section{Uncertainty Estimation}
\label{app:real_uncertainty}
\paragraph{Estimating data uncertainty.} Following \citet{kartashev2025unveiling}, we use a CatBoost model with dual outputs ($\mu(x)$ and $\sigma(x)$) trained via Gaussian negative log-likelihood on the real datasets, utilizing the predicted $\sigma(x)$ as our uncertainty estimate. To prevent data leakage, this estimator is fit exclusively on the training split and applied to validation and test points during inference. For the synthetic data, we use the ground-truth $\sigma(x)$.

  Figure~\ref{fig:real_uncertainty_more} extends the comparison with
  CatBoost to nine additional datasets. We use a single TabM \citep{gorishniy2025tabm}
  checkpoint with 32 branches. For each test input, we
  convert each branch's learned scale to a data variance using
  \eqref{eq:sigma_l} and average the resulting variances.
  Both methods' estimates are expressed in standardized-target
  variance units. The positive correlations provide further
  evidence that AdaQR learns meaningful input-dependent uncertainty.
 \begin{figure}[h]
      \centering
      \includegraphics[width=\linewidth]{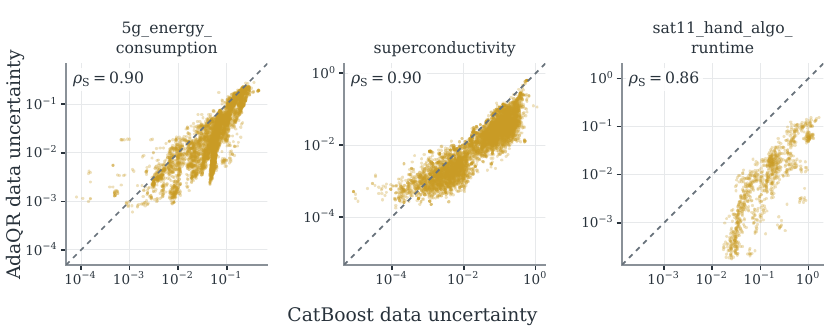}
      \par\vspace{2pt}
      \includegraphics[width=\linewidth]{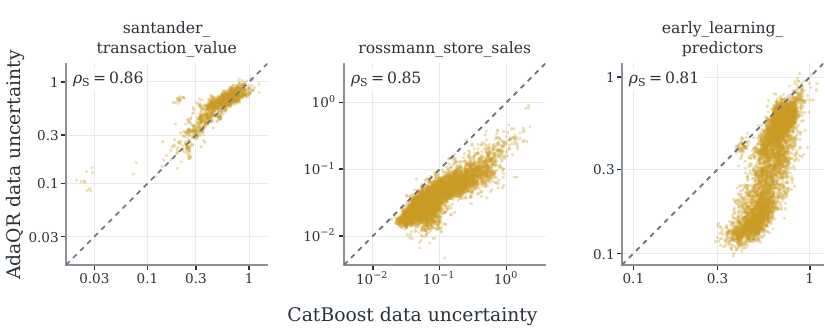}
      \par\vspace{2pt}
      \includegraphics[width=\linewidth]{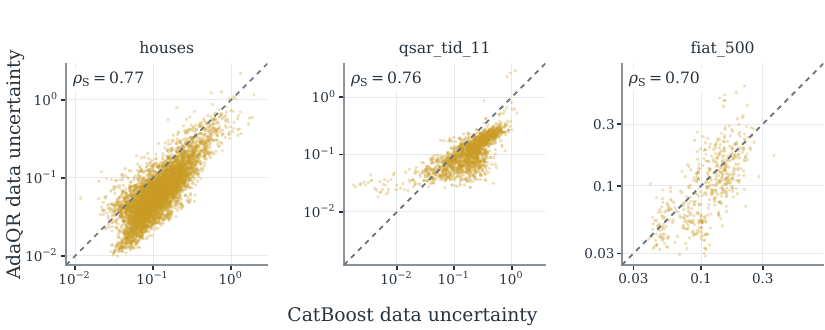}
      \caption{AdaQR versus CatBoost data uncertainty on nine
      additional real-world datasets. Points represent test inputs;
      AdaQR estimates are averaged across TabM branches after
      conversion to variances. Both axes are logarithmic, and dashed
      lines indicate equal estimates ($y=x$).
      $\rho_{\mathrm S}$ denotes Spearman's rank correlation,
      computed over all test inputs; at most 8,000 points are
      displayed per dataset.}
    \label{fig:real_uncertainty_more}
  \end{figure}

\section{Extended Results}
\label{app:extended_results}

In this section we present results from Section~\ref{sec:adaqr_results} with Tiny subset from BeyondArena \citep{purucker2026beyond} included (Table~\ref{tab:adasb-vs-quantile_w_tiny}), full results from cohort based early stopping section \ref{subsec:cohort_es} (Table~\ref{tab:cohort-early-stopping-dataset-wise}) and all per-dataset results used in this paper (Table~\ref{tab:extended-beyond-results}).

\begin{table}[h!]
\centering
\caption{Improv.~(\%) denotes the median relative improvement in RMSE compared to the standard MSE objective. W/T/L summarizes comparisons with MSE using the rule in
Appendix~\ref{app:comparison}.}
\label{tab:adasb-vs-quantile_w_tiny}
\begin{tabular}{llcc}
\toprule
\textbf{Model} & \textbf{Objective} & \textbf{Improv. (\%)} & \textbf{W / T / L} \\
\midrule
{MLP} & Quantile Loss & +0.55 & 21 / 15 / 5 \\
 & AdaQR & +0.74 & 19 / 17 / 5 \\
 & NLL & +0.25 & 10 / 23 / 8 \\
 & beta-NLL & +0.32 & 12 / 21 / 8 \\
\midrule
{TabPack} & Quantile Loss & +0.46 & 18 / 21 / 2 \\
 & AdaQR & +0.45 & 19 / 21 / 1 \\
\midrule
{RealMLP} & Quantile Loss & +0.40 & 20 / 16 / 5 \\
 & AdaQR$^{\dagger}$ & +0.90 & 25 / 9 / 7 \\
\midrule
{TabM} & Quantile Loss & +0.58 & 20 / 16 / 5 \\
 & AdaQR & +0.93 & 20 / 17 / 4 \\
\bottomrule
\end{tabular}
\end{table}

\begin{table*}[t]
\centering
\caption{Cohort-based vs. global checkpoint selection for an MLP
trained with quantile loss or AdaQR on the seven largest
BeyondArena datasets within our 41-dataset evaluation subset.
Values are mean $\pm$ standard deviation of test RMSE over
15 random seeds; lower is better.
Global denotes a single checkpoint shared by all test points;
$K=5$ and $K=10$ denote per-cohort selection with five and ten cohorts.
All checkpoints are selected on validation data.\vspace{5pt}}
\label{tab:cohort-early-stopping-dataset-wise}
\begingroup
\small
\setlength{\tabcolsep}{5pt}
\renewcommand{\arraystretch}{1.08}
\begin{tabular}{@{}llccc@{}}
\toprule
Objective & Cohorts & Global & $K=5$ & $K=10$ \\
\midrule
\multicolumn{5}{c}{\textbf{Electric Motor Temperature}} \\
\addlinespace[2pt]
Quantile & Uncertainty & $2.64 \pm 0.13$ & $2.65 \pm 0.13$ & $2.70 \pm 0.22$ \\
Quantile & Random & $2.64 \pm 0.13$ & $2.64 \pm 0.14$ & $2.64 \pm 0.13$ \\
AdaQR & Uncertainty & $2.74 \pm 0.16$ & $2.72 \pm 0.15$ & $2.73 \pm 0.14$ \\
\midrule
\multicolumn{5}{c}{\textbf{Mercari}} \\
\addlinespace[2pt]
Quantile & Uncertainty & $0.47298 \pm 0.00251$ & $0.47282 \pm 0.00181$ & $0.47296 \pm 0.00188$ \\
Quantile & Random & $0.47298 \pm 0.00251$ & $0.47289 \pm 0.00178$ & $0.47297 \pm 0.00182$ \\
\bottomrule
\end{tabular}
\endgroup
\end{table*}

\begin{table*}[p]
\centering
\caption{Dataset-wise MLP results with AdamW on 41 Beyond datasets (split 0). Test RMSE is mean $\pm$ sample standard deviation ($\mathrm{ddof}=1$). Single uses 15 seeds; E5 uses three disjoint ensembles of five seeds.}
\label{tab:extended-beyond-adamw-1}
\begin{minipage}[t]{0.49\textwidth}
\centering\tiny\setlength{\tabcolsep}{0.5pt}
\renewcommand{\arraystretch}{1.0}
\begin{tabular}{@{}p{0.24\textwidth}@{}p{0.38\textwidth}@{}p{0.38\textwidth}@{}}
\multicolumn{3}{c}{5g\_energy\_consumption} \\
\toprule
Objective & \makebox[0.38\textwidth][c]{Single} & \makebox[0.38\textwidth][c]{E5} \\
\midrule
MSE & \makebox[0.38\textwidth][c]{$4.10 \pm 0.02$} & \makebox[0.38\textwidth][c]{$4.054 \pm 0.008$} \\

HL-Gaussian & \makebox[0.38\textwidth][c]{$4.07 \pm 0.04$} & \makebox[0.38\textwidth][c]{$4.04 \pm 0.01$} \\
Quantile Loss & \makebox[0.38\textwidth][c]{$4.10 \pm 0.03$} & \makebox[0.38\textwidth][c]{$4.06 \pm 0.01$} \\
\bottomrule
\end{tabular}
\end{minipage}
\hfill
\begin{minipage}[t]{0.49\textwidth}
\centering\tiny\setlength{\tabcolsep}{0.5pt}
\renewcommand{\arraystretch}{1.0}
\begin{tabular}{@{}p{0.24\textwidth}@{}p{0.38\textwidth}@{}p{0.38\textwidth}@{}}
\multicolumn{3}{c}{airfoil\_self\_noise} \\
\toprule
Objective & \makebox[0.38\textwidth][c]{Single} & \makebox[0.38\textwidth][c]{E5} \\
\midrule
MSE & \makebox[0.38\textwidth][c]{$1.85 \pm 0.06$} & \makebox[0.38\textwidth][c]{$1.78 \pm 0.03$} \\
HL-Gaussian & \makebox[0.38\textwidth][c]{$1.88 \pm 0.04$} & \makebox[0.38\textwidth][c]{$1.84 \pm 0.02$} \\
Quantile Loss & \makebox[0.38\textwidth][c]{$1.91 \pm 0.03$} & \makebox[0.38\textwidth][c]{$1.87 \pm 0.02$} \\
\bottomrule
\end{tabular}
\end{minipage}
\par\vspace{3pt}
\begin{minipage}[t]{0.49\textwidth}
\centering\tiny\setlength{\tabcolsep}{0.5pt}
\renewcommand{\arraystretch}{1.0}
\begin{tabular}{@{}p{0.24\textwidth}@{}p{0.38\textwidth}@{}p{0.38\textwidth}@{}}
\multicolumn{3}{c}{allstate\_claims\_severity} \\
\toprule
Objective & \makebox[0.38\textwidth][c]{Single} & \makebox[0.38\textwidth][c]{E5} \\
\midrule
MSE & \makebox[0.38\textwidth][c]{$0.5408 \pm 0.0002$} & \makebox[0.38\textwidth][c]{$0.5401 \pm 0.0001$} \\
HL-Gaussian & \makebox[0.38\textwidth][c]{$0.5398 \pm 0.0001$} & \makebox[0.38\textwidth][c]{$0.5392 \pm 0.0000$} \\
Quantile Loss & \makebox[0.38\textwidth][c]{$0.5401 \pm 0.0002$} & \makebox[0.38\textwidth][c]{$0.5396 \pm 0.0001$} \\
\bottomrule
\end{tabular}
\end{minipage}
\hfill
\begin{minipage}[t]{0.49\textwidth}
\centering\tiny\setlength{\tabcolsep}{0.5pt}
\renewcommand{\arraystretch}{1.0}
\begin{tabular}{@{}p{0.24\textwidth}@{}p{0.38\textwidth}@{}p{0.38\textwidth}@{}}
\multicolumn{3}{c}{blood\_tests\_drink\_prediction} \\
\toprule
Objective & \makebox[0.38\textwidth][c]{Single} & \makebox[0.38\textwidth][c]{E5} \\
\midrule
MSE & \makebox[0.38\textwidth][c]{$0.68 \pm 0.01$} & \makebox[0.38\textwidth][c]{$0.680 \pm 0.010$} \\
HL-Gaussian & \makebox[0.38\textwidth][c]{$0.71 \pm 0.03$} & \makebox[0.38\textwidth][c]{$0.70 \pm 0.02$} \\
Quantile Loss & \makebox[0.38\textwidth][c]{$0.674 \pm 0.007$} & \makebox[0.38\textwidth][c]{$0.673 \pm 0.004$} \\
\bottomrule
\end{tabular}
\end{minipage}
\par\vspace{3pt}
\begin{minipage}[t]{0.49\textwidth}
\centering\tiny\setlength{\tabcolsep}{0.5pt}
\renewcommand{\arraystretch}{1.0}
\begin{tabular}{@{}p{0.24\textwidth}@{}p{0.38\textwidth}@{}p{0.38\textwidth}@{}}
\multicolumn{3}{c}{body\_density\_prediction} \\
\toprule
Objective & \makebox[0.38\textwidth][c]{Single} & \makebox[0.38\textwidth][c]{E5} \\
\midrule
MSE & \makebox[0.38\textwidth][c]{$0.0127 \pm 0.0005$} & \makebox[0.38\textwidth][c]{$0.0125 \pm 0.0003$} \\
HL-Gaussian & \makebox[0.38\textwidth][c]{$0.0130 \pm 0.0003$} & \makebox[0.38\textwidth][c]{$0.0129 \pm 0.0001$} \\
Quantile Loss & \makebox[0.38\textwidth][c]{$0.0128 \pm 0.0003$} & \makebox[0.38\textwidth][c]{$0.0127 \pm 0.0000$} \\
\bottomrule
\end{tabular}
\end{minipage}
\hfill
\begin{minipage}[t]{0.49\textwidth}
\centering\tiny\setlength{\tabcolsep}{0.5pt}
\renewcommand{\arraystretch}{1.0}
\begin{tabular}{@{}p{0.24\textwidth}@{}p{0.38\textwidth}@{}p{0.38\textwidth}@{}}
\multicolumn{3}{c}{california\_house\_prices\_2020} \\
\toprule
Objective & \makebox[0.38\textwidth][c]{Single} & \makebox[0.38\textwidth][c]{E5} \\
\midrule
MSE & \makebox[0.38\textwidth][c]{$0.280 \pm 0.006$} & \makebox[0.38\textwidth][c]{$0.271 \pm 0.001$} \\
HL-Gaussian & \makebox[0.38\textwidth][c]{$0.271 \pm 0.003$} & \makebox[0.38\textwidth][c]{$0.2643 \pm 0.0010$} \\
Quantile Loss & \makebox[0.38\textwidth][c]{$0.259 \pm 0.005$} & \makebox[0.38\textwidth][c]{$0.254 \pm 0.001$} \\
\bottomrule
\end{tabular}
\end{minipage}
\par\vspace{3pt}
\begin{minipage}[t]{0.49\textwidth}
\centering\tiny\setlength{\tabcolsep}{0.5pt}
\renewcommand{\arraystretch}{1.0}
\begin{tabular}{@{}p{0.24\textwidth}@{}p{0.38\textwidth}@{}p{0.38\textwidth}@{}}
\multicolumn{3}{c}{cirrhosis\_patient\_survival\_prediction} \\
\toprule
Objective & \makebox[0.38\textwidth][c]{Single} & \makebox[0.38\textwidth][c]{E5} \\
\midrule
MSE & \makebox[0.38\textwidth][c]{$0.71 \pm 0.02$} & \makebox[0.38\textwidth][c]{$0.697 \pm 0.009$} \\
HL-Gaussian & \makebox[0.38\textwidth][c]{$0.68 \pm 0.01$} & \makebox[0.38\textwidth][c]{$0.677 \pm 0.006$} \\
Quantile Loss & \makebox[0.38\textwidth][c]{$0.72 \pm 0.02$} & \makebox[0.38\textwidth][c]{$0.707 \pm 0.005$} \\
\bottomrule
\end{tabular}
\end{minipage}
\hfill
\begin{minipage}[t]{0.49\textwidth}
\centering\tiny\setlength{\tabcolsep}{0.5pt}
\renewcommand{\arraystretch}{1.0}
\begin{tabular}{@{}p{0.24\textwidth}@{}p{0.38\textwidth}@{}p{0.38\textwidth}@{}}
\multicolumn{3}{c}{climate\_model\_weather\_forecasting\_1m} \\
\toprule
Objective & \makebox[0.38\textwidth][c]{Single} & \makebox[0.38\textwidth][c]{E5} \\
\midrule
MSE & \makebox[0.38\textwidth][c]{$1.49 \pm 0.01$} & \makebox[0.38\textwidth][c]{$1.460 \pm 0.004$} \\
HL-Gaussian & \makebox[0.38\textwidth][c]{$1.48 \pm 0.02$} & \makebox[0.38\textwidth][c]{$1.470 \pm 0.008$} \\
Quantile Loss & \makebox[0.38\textwidth][c]{$1.47 \pm 0.01$} & \makebox[0.38\textwidth][c]{$1.453 \pm 0.008$} \\
\bottomrule
\end{tabular}
\end{minipage}
\par\vspace{3pt}
\begin{minipage}[t]{0.49\textwidth}
\centering\tiny\setlength{\tabcolsep}{0.5pt}
\renewcommand{\arraystretch}{1.0}
\begin{tabular}{@{}p{0.24\textwidth}@{}p{0.38\textwidth}@{}p{0.38\textwidth}@{}}
\multicolumn{3}{c}{coffee\_rating\_prediction} \\
\toprule
Objective & \makebox[0.38\textwidth][c]{Single} & \makebox[0.38\textwidth][c]{E5} \\
\midrule
MSE & \makebox[0.38\textwidth][c]{$1.52 \pm 0.05$} & \makebox[0.38\textwidth][c]{$1.50 \pm 0.03$} \\
HL-Gaussian & \makebox[0.38\textwidth][c]{$1.54 \pm 0.05$} & \makebox[0.38\textwidth][c]{$1.522 \pm 0.007$} \\
Quantile Loss & \makebox[0.38\textwidth][c]{$1.50 \pm 0.05$} & \makebox[0.38\textwidth][c]{$1.48 \pm 0.02$} \\
\bottomrule
\end{tabular}
\end{minipage}
\hfill
\begin{minipage}[t]{0.49\textwidth}
\centering\tiny\setlength{\tabcolsep}{0.5pt}
\renewcommand{\arraystretch}{1.0}
\begin{tabular}{@{}p{0.24\textwidth}@{}p{0.38\textwidth}@{}p{0.38\textwidth}@{}}
\multicolumn{3}{c}{concrete\_compressive\_strength} \\
\toprule
Objective & \makebox[0.38\textwidth][c]{Single} & \makebox[0.38\textwidth][c]{E5} \\
\midrule
MSE & \makebox[0.38\textwidth][c]{$4.9 \pm 0.1$} & \makebox[0.38\textwidth][c]{$4.75 \pm 0.04$} \\
HL-Gaussian & \makebox[0.38\textwidth][c]{$4.9 \pm 0.2$} & \makebox[0.38\textwidth][c]{$4.78 \pm 0.03$} \\
Quantile Loss & \makebox[0.38\textwidth][c]{$4.9 \pm 0.2$} & \makebox[0.38\textwidth][c]{$4.67 \pm 0.09$} \\
\bottomrule
\end{tabular}
\end{minipage}
\par\vspace{3pt}
\begin{minipage}[t]{0.49\textwidth}
\centering\tiny\setlength{\tabcolsep}{0.5pt}
\renewcommand{\arraystretch}{1.0}
\begin{tabular}{@{}p{0.24\textwidth}@{}p{0.38\textwidth}@{}p{0.38\textwidth}@{}}
\multicolumn{3}{c}{diamonds} \\
\toprule
Objective & \makebox[0.38\textwidth][c]{Single} & \makebox[0.38\textwidth][c]{E5} \\
\midrule
MSE & \makebox[0.38\textwidth][c]{$0.093 \pm 0.002$} & \makebox[0.38\textwidth][c]{$0.0894 \pm 0.0005$} \\
HL-Gaussian & \makebox[0.38\textwidth][c]{$0.0877 \pm 0.0005$} & \makebox[0.38\textwidth][c]{$0.0861 \pm 0.0001$} \\
Quantile Loss & \makebox[0.38\textwidth][c]{$0.093 \pm 0.002$} & \makebox[0.38\textwidth][c]{$0.089 \pm 0.001$} \\
\bottomrule
\end{tabular}
\end{minipage}
\hfill
\begin{minipage}[t]{0.49\textwidth}
\centering\tiny\setlength{\tabcolsep}{0.5pt}
\renewcommand{\arraystretch}{1.0}
\begin{tabular}{@{}p{0.24\textwidth}@{}p{0.38\textwidth}@{}p{0.38\textwidth}@{}}
\multicolumn{3}{c}{early\_learning\_predictors} \\
\toprule
Objective & \makebox[0.38\textwidth][c]{Single} & \makebox[0.38\textwidth][c]{E5} \\
\midrule
MSE & \makebox[0.38\textwidth][c]{$13.44 \pm 0.04$} & \makebox[0.38\textwidth][c]{$13.375 \pm 0.008$} \\
HL-Gaussian & \makebox[0.38\textwidth][c]{$13.36 \pm 0.03$} & \makebox[0.38\textwidth][c]{$13.28 \pm 0.01$} \\
Quantile Loss & \makebox[0.38\textwidth][c]{$13.40 \pm 0.07$} & \makebox[0.38\textwidth][c]{$13.313 \pm 0.009$} \\
\bottomrule
\end{tabular}
\end{minipage}
\par\vspace{3pt}
\begin{minipage}[t]{0.49\textwidth}
\centering\tiny\setlength{\tabcolsep}{0.5pt}
\renewcommand{\arraystretch}{1.0}
\begin{tabular}{@{}p{0.24\textwidth}@{}p{0.38\textwidth}@{}p{0.38\textwidth}@{}}
\multicolumn{3}{c}{electric\_motor\_temperature\_prediction} \\
\toprule
Objective & \makebox[0.38\textwidth][c]{Single} & \makebox[0.38\textwidth][c]{E5} \\
\midrule
MSE & \makebox[0.38\textwidth][c]{$3.0 \pm 0.1$} & \makebox[0.38\textwidth][c]{$2.88 \pm 0.07$} \\
HL-Gaussian & \makebox[0.38\textwidth][c]{$2.78 \pm 0.07$} & \makebox[0.38\textwidth][c]{$2.70 \pm 0.01$} \\
Quantile Loss & \makebox[0.38\textwidth][c]{$3.1 \pm 0.2$} & \makebox[0.38\textwidth][c]{$2.86 \pm 0.09$} \\
\bottomrule
\end{tabular}
\end{minipage}
\hfill
\begin{minipage}[t]{0.49\textwidth}
\centering\tiny\setlength{\tabcolsep}{0.5pt}
\renewcommand{\arraystretch}{1.0}
\begin{tabular}{@{}p{0.24\textwidth}@{}p{0.38\textwidth}@{}p{0.38\textwidth}@{}}
\multicolumn{3}{c}{fiat\_500} \\
\toprule
Objective & \makebox[0.38\textwidth][c]{Single} & \makebox[0.38\textwidth][c]{E5} \\
\midrule
MSE & \makebox[0.38\textwidth][c]{$891 \pm 18$} & \makebox[0.38\textwidth][c]{$885 \pm 8$} \\
HL-Gaussian & \makebox[0.38\textwidth][c]{$822 \pm 8$} & \makebox[0.38\textwidth][c]{$816.3 \pm 0.8$} \\
Quantile Loss & \makebox[0.38\textwidth][c]{$878 \pm 16$} & \makebox[0.38\textwidth][c]{$871 \pm 3$} \\
\bottomrule
\end{tabular}
\end{minipage}
\par\vspace{3pt}
\begin{minipage}[t]{0.49\textwidth}
\centering\tiny\setlength{\tabcolsep}{0.5pt}
\renewcommand{\arraystretch}{1.0}
\begin{tabular}{@{}p{0.24\textwidth}@{}p{0.38\textwidth}@{}p{0.38\textwidth}@{}}
\multicolumn{3}{c}{food\_delivery\_time} \\
\toprule
Objective & \makebox[0.38\textwidth][c]{Single} & \makebox[0.38\textwidth][c]{E5} \\
\midrule
MSE & \makebox[0.38\textwidth][c]{$8.07 \pm 0.01$} & \makebox[0.38\textwidth][c]{$8.036 \pm 0.005$} \\
HL-Gaussian & \makebox[0.38\textwidth][c]{$8.12 \pm 0.01$} & \makebox[0.38\textwidth][c]{$8.100 \pm 0.005$} \\
Quantile Loss & \makebox[0.38\textwidth][c]{$8.10 \pm 0.01$} & \makebox[0.38\textwidth][c]{$8.079 \pm 0.005$} \\
\bottomrule
\end{tabular}
\end{minipage}
\hfill
\begin{minipage}[t]{0.49\textwidth}
\centering\tiny\setlength{\tabcolsep}{0.5pt}
\renewcommand{\arraystretch}{1.0}
\begin{tabular}{@{}p{0.24\textwidth}@{}p{0.38\textwidth}@{}p{0.38\textwidth}@{}}
\multicolumn{3}{c}{forest\_fires} \\
\toprule
Objective & \makebox[0.38\textwidth][c]{Single} & \makebox[0.38\textwidth][c]{E5} \\
\midrule
MSE & \makebox[0.38\textwidth][c]{$1.37 \pm 0.01$} & \makebox[0.38\textwidth][c]{$1.361 \pm 0.002$} \\
HL-Gaussian & \makebox[0.38\textwidth][c]{$1.356 \pm 0.009$} & \makebox[0.38\textwidth][c]{$1.3546 \pm 0.0001$} \\
Quantile Loss & \makebox[0.38\textwidth][c]{$1.351 \pm 0.003$} & \makebox[0.38\textwidth][c]{$1.3496 \pm 0.0006$} \\
\bottomrule
\end{tabular}
\end{minipage}
\par\vspace{3pt}
\begin{minipage}[t]{0.49\textwidth}
\centering\tiny\setlength{\tabcolsep}{0.5pt}
\renewcommand{\arraystretch}{1.0}
\begin{tabular}{@{}p{0.24\textwidth}@{}p{0.38\textwidth}@{}p{0.38\textwidth}@{}}
\multicolumn{3}{c}{garments\_worker\_productivity} \\
\toprule
Objective & \makebox[0.38\textwidth][c]{Single} & \makebox[0.38\textwidth][c]{E5} \\
\midrule
MSE & \makebox[0.38\textwidth][c]{$0.101 \pm 0.009$} & \makebox[0.38\textwidth][c]{$0.096 \pm 0.002$} \\
HL-Gaussian & \makebox[0.38\textwidth][c]{$0.095 \pm 0.003$} & \makebox[0.38\textwidth][c]{$0.093 \pm 0.001$} \\
Quantile Loss & \makebox[0.38\textwidth][c]{$0.096 \pm 0.007$} & \makebox[0.38\textwidth][c]{$0.090 \pm 0.001$} \\
\bottomrule
\end{tabular}
\end{minipage}
\hfill
\begin{minipage}[t]{0.49\textwidth}
\centering\tiny\setlength{\tabcolsep}{0.5pt}
\renewcommand{\arraystretch}{1.0}
\begin{tabular}{@{}p{0.24\textwidth}@{}p{0.38\textwidth}@{}p{0.38\textwidth}@{}}
\multicolumn{3}{c}{healthcare\_insurance\_expenses} \\
\toprule
Objective & \makebox[0.38\textwidth][c]{Single} & \makebox[0.38\textwidth][c]{E5} \\
\midrule
MSE & \makebox[0.38\textwidth][c]{$0.401 \pm 0.002$} & \makebox[0.38\textwidth][c]{$0.399 \pm 0.001$} \\
HL-Gaussian & \makebox[0.38\textwidth][c]{$0.398 \pm 0.004$} & \makebox[0.38\textwidth][c]{$0.395 \pm 0.002$} \\
Quantile Loss & \makebox[0.38\textwidth][c]{$0.389 \pm 0.002$} & \makebox[0.38\textwidth][c]{$0.3863 \pm 0.0009$} \\
\bottomrule
\end{tabular}
\end{minipage}
\par\vspace{3pt}
\begin{minipage}[t]{0.49\textwidth}
\centering\tiny\setlength{\tabcolsep}{0.5pt}
\renewcommand{\arraystretch}{1.0}
\begin{tabular}{@{}p{0.24\textwidth}@{}p{0.38\textwidth}@{}p{0.38\textwidth}@{}}
\multicolumn{3}{c}{houses} \\
\toprule
Objective & \makebox[0.38\textwidth][c]{Single} & \makebox[0.38\textwidth][c]{E5} \\
\midrule
MSE & \makebox[0.38\textwidth][c]{$0.229 \pm 0.002$} & \makebox[0.38\textwidth][c]{$0.2252 \pm 0.0007$} \\
HL-Gaussian & \makebox[0.38\textwidth][c]{$0.227 \pm 0.002$} & \makebox[0.38\textwidth][c]{$0.224 \pm 0.001$} \\
Quantile Loss & \makebox[0.38\textwidth][c]{$0.228 \pm 0.002$} & \makebox[0.38\textwidth][c]{$0.2245 \pm 0.0007$} \\
\bottomrule
\end{tabular}
\end{minipage}
\hfill
\begin{minipage}[t]{0.49\textwidth}
\centering\tiny\setlength{\tabcolsep}{0.5pt}
\renewcommand{\arraystretch}{1.0}
\begin{tabular}{@{}p{0.24\textwidth}@{}p{0.38\textwidth}@{}p{0.38\textwidth}@{}}
\multicolumn{3}{c}{immoscout\_german\_house\_prices} \\
\toprule
Objective & \makebox[0.38\textwidth][c]{Single} & \makebox[0.38\textwidth][c]{E5} \\
\midrule
MSE & \makebox[0.38\textwidth][c]{$0.406 \pm 0.003$} & \makebox[0.38\textwidth][c]{$0.3977 \pm 0.0008$} \\
HL-Gaussian & \makebox[0.38\textwidth][c]{$0.407 \pm 0.002$} & \makebox[0.38\textwidth][c]{$0.401 \pm 0.001$} \\
Quantile Loss & \makebox[0.38\textwidth][c]{$0.406 \pm 0.003$} & \makebox[0.38\textwidth][c]{$0.397 \pm 0.001$} \\
\bottomrule
\end{tabular}
\end{minipage}
\par\vspace{3pt}
\begin{minipage}[t]{0.49\textwidth}
\centering\tiny\setlength{\tabcolsep}{0.5pt}
\renewcommand{\arraystretch}{1.0}
\begin{tabular}{@{}p{0.24\textwidth}@{}p{0.38\textwidth}@{}p{0.38\textwidth}@{}}
\multicolumn{3}{c}{mercari\_price\_suggestion\_1m} \\
\toprule
Objective & \makebox[0.38\textwidth][c]{Single} & \makebox[0.38\textwidth][c]{E5} \\
\midrule
MSE & \makebox[0.38\textwidth][c]{$0.4756 \pm 0.0005$} & \makebox[0.38\textwidth][c]{$0.4635 \pm 0.0008$} \\
HL-Gaussian & \makebox[0.38\textwidth][c]{$0.4731 \pm 0.0005$} & \makebox[0.38\textwidth][c]{$0.4619 \pm 0.0000$} \\
Quantile Loss & \makebox[0.38\textwidth][c]{$0.4752 \pm 0.0007$} & \makebox[0.38\textwidth][c]{$0.4625 \pm 0.0003$} \\
\bottomrule
\end{tabular}
\end{minipage}
\hfill\mbox{}
\end{table*}

\begin{table*}[p]
\centering
\caption{Dataset-wise MLP results with AdamW (continued). Test RMSE: mean $\pm$ sample standard deviation. Single: 15 seeds; E5: three ensembles of five.}
\label{tab:extended-beyond-adamw-2}
\begin{minipage}[t]{0.49\textwidth}
\centering\tiny\setlength{\tabcolsep}{0.5pt}
\renewcommand{\arraystretch}{1.0}
\begin{tabular}{@{}p{0.24\textwidth}@{}p{0.38\textwidth}@{}p{0.38\textwidth}@{}}
\multicolumn{3}{c}{mercedes\_benz\_greener\_manufacturing} \\
\toprule
Objective & \makebox[0.38\textwidth][c]{Single} & \makebox[0.38\textwidth][c]{E5} \\
\midrule
MSE & \makebox[0.38\textwidth][c]{$8.45 \pm 0.07$} & \makebox[0.38\textwidth][c]{$8.42 \pm 0.02$} \\
HL-Gaussian & \makebox[0.38\textwidth][c]{$8.39 \pm 0.07$} & \makebox[0.38\textwidth][c]{$8.34 \pm 0.03$} \\
Quantile Loss & \makebox[0.38\textwidth][c]{$8.34 \pm 0.03$} & \makebox[0.38\textwidth][c]{$8.325 \pm 0.009$} \\
\bottomrule
\end{tabular}
\end{minipage}
\hfill
\begin{minipage}[t]{0.49\textwidth}
\centering\tiny\setlength{\tabcolsep}{0.5pt}
\renewcommand{\arraystretch}{1.0}
\begin{tabular}{@{}p{0.24\textwidth}@{}p{0.38\textwidth}@{}p{0.38\textwidth}@{}}
\multicolumn{3}{c}{miami\_housing} \\
\toprule
Objective & \makebox[0.38\textwidth][c]{Single} & \makebox[0.38\textwidth][c]{E5} \\
\midrule
MSE & \makebox[0.38\textwidth][c]{$0.1498 \pm 0.0009$} & \makebox[0.38\textwidth][c]{$0.1482 \pm 0.0004$} \\
HL-Gaussian & \makebox[0.38\textwidth][c]{$0.149 \pm 0.001$} & \makebox[0.38\textwidth][c]{$0.1465 \pm 0.0008$} \\
Quantile Loss & \makebox[0.38\textwidth][c]{$0.1488 \pm 0.0008$} & \makebox[0.38\textwidth][c]{$0.1472 \pm 0.0003$} \\
\bottomrule
\end{tabular}
\end{minipage}
\par\vspace{3pt}
\begin{minipage}[t]{0.49\textwidth}
\centering\tiny\setlength{\tabcolsep}{0.5pt}
\renewcommand{\arraystretch}{1.0}
\begin{tabular}{@{}p{0.24\textwidth}@{}p{0.38\textwidth}@{}p{0.38\textwidth}@{}}
\multicolumn{3}{c}{mutual\_funds\_india} \\
\toprule
Objective & \makebox[0.38\textwidth][c]{Single} & \makebox[0.38\textwidth][c]{E5} \\
\midrule
MSE & \makebox[0.38\textwidth][c]{$6.0 \pm 0.1$} & \makebox[0.38\textwidth][c]{$5.83 \pm 0.08$} \\
HL-Gaussian & \makebox[0.38\textwidth][c]{$5.9 \pm 0.1$} & \makebox[0.38\textwidth][c]{$5.72 \pm 0.03$} \\
Quantile Loss & \makebox[0.38\textwidth][c]{$5.9 \pm 0.1$} & \makebox[0.38\textwidth][c]{$5.78 \pm 0.03$} \\
\bottomrule
\end{tabular}
\end{minipage}
\hfill
\begin{minipage}[t]{0.49\textwidth}
\centering\tiny\setlength{\tabcolsep}{0.5pt}
\renewcommand{\arraystretch}{1.0}
\begin{tabular}{@{}p{0.24\textwidth}@{}p{0.38\textwidth}@{}p{0.38\textwidth}@{}}
\multicolumn{3}{c}{obesity\_estimation} \\
\toprule
Objective & \makebox[0.38\textwidth][c]{Single} & \makebox[0.38\textwidth][c]{E5} \\
\midrule
MSE & \makebox[0.38\textwidth][c]{$4.94 \pm 0.02$} & \makebox[0.38\textwidth][c]{$4.919 \pm 0.007$} \\
HL-Gaussian & \makebox[0.38\textwidth][c]{$4.97 \pm 0.04$} & \makebox[0.38\textwidth][c]{$4.94 \pm 0.02$} \\
Quantile Loss & \makebox[0.38\textwidth][c]{$4.90 \pm 0.02$} & \makebox[0.38\textwidth][c]{$4.899 \pm 0.006$} \\
\bottomrule
\end{tabular}
\end{minipage}
\par\vspace{3pt}
\begin{minipage}[t]{0.49\textwidth}
\centering\tiny\setlength{\tabcolsep}{0.5pt}
\renewcommand{\arraystretch}{1.0}
\begin{tabular}{@{}p{0.24\textwidth}@{}p{0.38\textwidth}@{}p{0.38\textwidth}@{}}
\multicolumn{3}{c}{physiochemical\_protein} \\
\toprule
Objective & \makebox[0.38\textwidth][c]{Single} & \makebox[0.38\textwidth][c]{E5} \\
\midrule
MSE & \makebox[0.38\textwidth][c]{$3.52 \pm 0.01$} & \makebox[0.38\textwidth][c]{$3.431 \pm 0.006$} \\
HL-Gaussian & \makebox[0.38\textwidth][c]{$3.43 \pm 0.02$} & \makebox[0.38\textwidth][c]{$3.35 \pm 0.02$} \\
Quantile Loss & \makebox[0.38\textwidth][c]{$3.52 \pm 0.03$} & \makebox[0.38\textwidth][c]{$3.43 \pm 0.02$} \\
\bottomrule
\end{tabular}
\end{minipage}
\hfill
\begin{minipage}[t]{0.49\textwidth}
\centering\tiny\setlength{\tabcolsep}{0.5pt}
\renewcommand{\arraystretch}{1.0}
\begin{tabular}{@{}p{0.24\textwidth}@{}p{0.38\textwidth}@{}p{0.38\textwidth}@{}}
\multicolumn{3}{c}{qsar\_aquatic\_toxicity} \\
\toprule
Objective & \makebox[0.38\textwidth][c]{Single} & \makebox[0.38\textwidth][c]{E5} \\
\midrule
MSE & \makebox[0.38\textwidth][c]{$1.16 \pm 0.02$} & \makebox[0.38\textwidth][c]{$1.149 \pm 0.009$} \\
HL-Gaussian & \makebox[0.38\textwidth][c]{$1.18 \pm 0.02$} & \makebox[0.38\textwidth][c]{$1.17 \pm 0.01$} \\
Quantile Loss & \makebox[0.38\textwidth][c]{$1.18 \pm 0.02$} & \makebox[0.38\textwidth][c]{$1.173 \pm 0.003$} \\
\bottomrule
\end{tabular}
\end{minipage}
\par\vspace{3pt}
\begin{minipage}[t]{0.49\textwidth}
\centering\tiny\setlength{\tabcolsep}{0.5pt}
\renewcommand{\arraystretch}{1.0}
\begin{tabular}{@{}p{0.24\textwidth}@{}p{0.38\textwidth}@{}p{0.38\textwidth}@{}}
\multicolumn{3}{c}{qsar\_fish\_toxicity} \\
\toprule
Objective & \makebox[0.38\textwidth][c]{Single} & \makebox[0.38\textwidth][c]{E5} \\
\midrule
MSE & \makebox[0.38\textwidth][c]{$0.95 \pm 0.02$} & \makebox[0.38\textwidth][c]{$0.938 \pm 0.007$} \\
HL-Gaussian & \makebox[0.38\textwidth][c]{$0.93 \pm 0.02$} & \makebox[0.38\textwidth][c]{$0.905 \pm 0.003$} \\
Quantile Loss & \makebox[0.38\textwidth][c]{$0.973 \pm 0.010$} & \makebox[0.38\textwidth][c]{$0.968 \pm 0.003$} \\
\bottomrule
\end{tabular}
\end{minipage}
\hfill
\begin{minipage}[t]{0.49\textwidth}
\centering\tiny\setlength{\tabcolsep}{0.5pt}
\renewcommand{\arraystretch}{1.0}
\begin{tabular}{@{}p{0.24\textwidth}@{}p{0.38\textwidth}@{}p{0.38\textwidth}@{}}
\multicolumn{3}{c}{qsar\_tid\_11} \\
\toprule
Objective & \makebox[0.38\textwidth][c]{Single} & \makebox[0.38\textwidth][c]{E5} \\
\midrule
MSE & \makebox[0.38\textwidth][c]{$0.813 \pm 0.008$} & \makebox[0.38\textwidth][c]{$0.797 \pm 0.004$} \\
HL-Gaussian & \makebox[0.38\textwidth][c]{$0.804 \pm 0.005$} & \makebox[0.38\textwidth][c]{$0.792 \pm 0.003$} \\
Quantile Loss & \makebox[0.38\textwidth][c]{$0.811 \pm 0.007$} & \makebox[0.38\textwidth][c]{$0.795 \pm 0.002$} \\
\bottomrule
\end{tabular}
\end{minipage}
\par\vspace{3pt}
\begin{minipage}[t]{0.49\textwidth}
\centering\tiny\setlength{\tabcolsep}{0.5pt}
\renewcommand{\arraystretch}{1.0}
\begin{tabular}{@{}p{0.24\textwidth}@{}p{0.38\textwidth}@{}p{0.38\textwidth}@{}}
\multicolumn{3}{c}{rossmann\_store\_sales} \\
\toprule
Objective & \makebox[0.38\textwidth][c]{Single} & \makebox[0.38\textwidth][c]{E5} \\
\midrule
MSE & \makebox[0.38\textwidth][c]{$1474 \pm 126$} & \makebox[0.38\textwidth][c]{$1449 \pm 84$} \\
HL-Gaussian & \makebox[0.38\textwidth][c]{$1420 \pm 37$} & \makebox[0.38\textwidth][c]{$1414 \pm 11$} \\
Quantile Loss & \makebox[0.38\textwidth][c]{$1187 \pm 53$} & \makebox[0.38\textwidth][c]{$1134 \pm 22$} \\
\bottomrule
\end{tabular}
\end{minipage}
\hfill
\begin{minipage}[t]{0.49\textwidth}
\centering\tiny\setlength{\tabcolsep}{0.5pt}
\renewcommand{\arraystretch}{1.0}
\begin{tabular}{@{}p{0.24\textwidth}@{}p{0.38\textwidth}@{}p{0.38\textwidth}@{}}
\multicolumn{3}{c}{santander\_transaction\_value} \\
\toprule
Objective & \makebox[0.38\textwidth][c]{Single} & \makebox[0.38\textwidth][c]{E5} \\
\midrule
MSE & \makebox[0.38\textwidth][c]{$1.43 \pm 0.02$} & \makebox[0.38\textwidth][c]{$1.399 \pm 0.004$} \\
HL-Gaussian & \makebox[0.38\textwidth][c]{$1.407 \pm 0.003$} & \makebox[0.38\textwidth][c]{$1.406 \pm 0.001$} \\
Quantile Loss & \makebox[0.38\textwidth][c]{$1.410 \pm 0.004$} & \makebox[0.38\textwidth][c]{$1.406 \pm 0.002$} \\
\bottomrule
\end{tabular}
\end{minipage}
\par\vspace{3pt}
\begin{minipage}[t]{0.49\textwidth}
\centering\tiny\setlength{\tabcolsep}{0.5pt}
\renewcommand{\arraystretch}{1.0}
\begin{tabular}{@{}p{0.24\textwidth}@{}p{0.38\textwidth}@{}p{0.38\textwidth}@{}}
\multicolumn{3}{c}{sat11\_hand\_algo\_runtime} \\
\toprule
Objective & \makebox[0.38\textwidth][c]{Single} & \makebox[0.38\textwidth][c]{E5} \\
\midrule
MSE & \makebox[0.38\textwidth][c]{$2.1 \pm 0.2$} & \makebox[0.38\textwidth][c]{$1.99 \pm 0.03$} \\
HL-Gaussian & \makebox[0.38\textwidth][c]{$2.05 \pm 0.05$} & \makebox[0.38\textwidth][c]{$1.96 \pm 0.01$} \\
Quantile Loss & \makebox[0.38\textwidth][c]{$2.2 \pm 0.1$} & \makebox[0.38\textwidth][c]{$2.07 \pm 0.07$} \\
\bottomrule
\end{tabular}
\end{minipage}
\hfill
\begin{minipage}[t]{0.49\textwidth}
\centering\tiny\setlength{\tabcolsep}{0.5pt}
\renewcommand{\arraystretch}{1.0}
\begin{tabular}{@{}p{0.24\textwidth}@{}p{0.38\textwidth}@{}p{0.38\textwidth}@{}}
\multicolumn{3}{c}{sberbank\_housing\_market\_forecasting} \\
\toprule
Objective & \makebox[0.38\textwidth][c]{Single} & \makebox[0.38\textwidth][c]{E5} \\
\midrule
MSE & \makebox[0.38\textwidth][c]{$0.26 \pm 0.01$} & \makebox[0.38\textwidth][c]{$0.253 \pm 0.002$} \\
HL-Gaussian & \makebox[0.38\textwidth][c]{$0.253 \pm 0.003$} & \makebox[0.38\textwidth][c]{$0.250 \pm 0.002$} \\
Quantile Loss & \makebox[0.38\textwidth][c]{$0.256 \pm 0.009$} & \makebox[0.38\textwidth][c]{$0.246 \pm 0.002$} \\
\bottomrule
\end{tabular}
\end{minipage}
\par\vspace{3pt}
\begin{minipage}[t]{0.49\textwidth}
\centering\tiny\setlength{\tabcolsep}{0.5pt}
\renewcommand{\arraystretch}{1.0}
\begin{tabular}{@{}p{0.24\textwidth}@{}p{0.38\textwidth}@{}p{0.38\textwidth}@{}}
\multicolumn{3}{c}{sf\_permit\_time} \\
\toprule
Objective & \makebox[0.38\textwidth][c]{Single} & \makebox[0.38\textwidth][c]{E5} \\
\midrule
MSE & \makebox[0.38\textwidth][c]{$1.45 \pm 0.02$} & \makebox[0.38\textwidth][c]{$1.45 \pm 0.01$} \\
HL-Gaussian & \makebox[0.38\textwidth][c]{$1.46 \pm 0.01$} & \makebox[0.38\textwidth][c]{$1.460 \pm 0.004$} \\
Quantile Loss & \makebox[0.38\textwidth][c]{$1.44 \pm 0.01$} & \makebox[0.38\textwidth][c]{$1.440 \pm 0.003$} \\
\bottomrule
\end{tabular}
\end{minipage}
\hfill
\begin{minipage}[t]{0.49\textwidth}
\centering\tiny\setlength{\tabcolsep}{0.5pt}
\renewcommand{\arraystretch}{1.0}
\begin{tabular}{@{}p{0.24\textwidth}@{}p{0.38\textwidth}@{}p{0.38\textwidth}@{}}
\multicolumn{3}{c}{student\_portuguese\_performance} \\
\toprule
Objective & \makebox[0.38\textwidth][c]{Single} & \makebox[0.38\textwidth][c]{E5} \\
\midrule
MSE & \makebox[0.38\textwidth][c]{$3.04 \pm 0.05$} & \makebox[0.38\textwidth][c]{$2.98 \pm 0.04$} \\
HL-Gaussian & \makebox[0.38\textwidth][c]{$2.93 \pm 0.04$} & \makebox[0.38\textwidth][c]{$2.893 \pm 0.006$} \\
Quantile Loss & \makebox[0.38\textwidth][c]{$2.98 \pm 0.04$} & \makebox[0.38\textwidth][c]{$2.929 \pm 0.009$} \\
\bottomrule
\end{tabular}
\end{minipage}
\par\vspace{3pt}
\begin{minipage}[t]{0.49\textwidth}
\centering\tiny\setlength{\tabcolsep}{0.5pt}
\renewcommand{\arraystretch}{1.0}
\begin{tabular}{@{}p{0.24\textwidth}@{}p{0.38\textwidth}@{}p{0.38\textwidth}@{}}
\multicolumn{3}{c}{superconductivity} \\
\toprule
Objective & \makebox[0.38\textwidth][c]{Single} & \makebox[0.38\textwidth][c]{E5} \\
\midrule
MSE & \makebox[0.38\textwidth][c]{$10.2 \pm 0.1$} & \makebox[0.38\textwidth][c]{$9.92 \pm 0.03$} \\
HL-Gaussian & \makebox[0.38\textwidth][c]{$10.2 \pm 0.1$} & \makebox[0.38\textwidth][c]{$9.85 \pm 0.03$} \\
Quantile Loss & \makebox[0.38\textwidth][c]{$10.0 \pm 0.1$} & \makebox[0.38\textwidth][c]{$9.81 \pm 0.06$} \\
\bottomrule
\end{tabular}
\end{minipage}
\hfill
\begin{minipage}[t]{0.49\textwidth}
\centering\tiny\setlength{\tabcolsep}{0.5pt}
\renewcommand{\arraystretch}{1.0}
\begin{tabular}{@{}p{0.24\textwidth}@{}p{0.38\textwidth}@{}p{0.38\textwidth}@{}}
\multicolumn{3}{c}{telemonitoring\_parkinsons\_biomedical\_voice\_measurements} \\
\toprule
Objective & \makebox[0.38\textwidth][c]{Single} & \makebox[0.38\textwidth][c]{E5} \\
\midrule
MSE & \makebox[0.38\textwidth][c]{$11 \pm 2$} & \makebox[0.38\textwidth][c]{$11.0 \pm 0.7$} \\
HL-Gaussian & \makebox[0.38\textwidth][c]{$10.9 \pm 0.7$} & \makebox[0.38\textwidth][c]{$10.4 \pm 0.4$} \\
Quantile Loss & \makebox[0.38\textwidth][c]{$10.6 \pm 0.6$} & \makebox[0.38\textwidth][c]{$10.2 \pm 0.3$} \\
\bottomrule
\end{tabular}
\end{minipage}
\par\vspace{3pt}
\begin{minipage}[t]{0.49\textwidth}
\centering\tiny\setlength{\tabcolsep}{0.5pt}
\renewcommand{\arraystretch}{1.0}
\begin{tabular}{@{}p{0.24\textwidth}@{}p{0.38\textwidth}@{}p{0.38\textwidth}@{}}
\multicolumn{3}{c}{video\_game\_fps\_prediction} \\
\toprule
Objective & \makebox[0.38\textwidth][c]{Single} & \makebox[0.38\textwidth][c]{E5} \\
\midrule
MSE & \makebox[0.38\textwidth][c]{$0.0145 \pm 0.0009$} & \makebox[0.38\textwidth][c]{$0.0136 \pm 0.0004$} \\
HL-Gaussian & \makebox[0.38\textwidth][c]{$0.0154 \pm 0.0010$} & \makebox[0.38\textwidth][c]{$0.0151 \pm 0.0004$} \\
Quantile Loss & \makebox[0.38\textwidth][c]{$0.012 \pm 0.001$} & \makebox[0.38\textwidth][c]{$0.0101 \pm 0.0004$} \\
\bottomrule
\end{tabular}
\end{minipage}
\hfill
\begin{minipage}[t]{0.49\textwidth}
\centering\tiny\setlength{\tabcolsep}{0.5pt}
\renewcommand{\arraystretch}{1.0}
\begin{tabular}{@{}p{0.24\textwidth}@{}p{0.38\textwidth}@{}p{0.38\textwidth}@{}}
\multicolumn{3}{c}{video\_transcoding\_time\_prediction} \\
\toprule
Objective & \makebox[0.38\textwidth][c]{Single} & \makebox[0.38\textwidth][c]{E5} \\
\midrule
MSE & \makebox[0.38\textwidth][c]{$0.273 \pm 0.007$} & \makebox[0.38\textwidth][c]{$0.263 \pm 0.005$} \\
HL-Gaussian & \makebox[0.38\textwidth][c]{$0.264 \pm 0.007$} & \makebox[0.38\textwidth][c]{$0.258 \pm 0.002$} \\
Quantile Loss & \makebox[0.38\textwidth][c]{$0.264 \pm 0.005$} & \makebox[0.38\textwidth][c]{$0.259 \pm 0.002$} \\
\bottomrule
\end{tabular}
\end{minipage}
\par\vspace{3pt}
\begin{minipage}[t]{0.49\textwidth}
\centering\tiny\setlength{\tabcolsep}{0.5pt}
\renewcommand{\arraystretch}{1.0}
\begin{tabular}{@{}p{0.24\textwidth}@{}p{0.38\textwidth}@{}p{0.38\textwidth}@{}}
\multicolumn{3}{c}{wine\_quality} \\
\toprule
Objective & \makebox[0.38\textwidth][c]{Single} & \makebox[0.38\textwidth][c]{E5} \\
\midrule
MSE & \makebox[0.38\textwidth][c]{$0.677 \pm 0.002$} & \makebox[0.38\textwidth][c]{$0.672 \pm 0.001$} \\
HL-Gaussian & \makebox[0.38\textwidth][c]{$0.670 \pm 0.003$} & \makebox[0.38\textwidth][c]{$0.665 \pm 0.002$} \\
Quantile Loss & \makebox[0.38\textwidth][c]{$0.674 \pm 0.003$} & \makebox[0.38\textwidth][c]{$0.6655 \pm 0.0003$} \\
\bottomrule
\end{tabular}
\end{minipage}
\hfill
\begin{minipage}[t]{0.49\textwidth}
\centering\tiny\setlength{\tabcolsep}{0.5pt}
\renewcommand{\arraystretch}{1.0}
\begin{tabular}{@{}p{0.24\textwidth}@{}p{0.38\textwidth}@{}p{0.38\textwidth}@{}}
\multicolumn{3}{c}{wine\_world\_cost} \\
\toprule
Objective & \makebox[0.38\textwidth][c]{Single} & \makebox[0.38\textwidth][c]{E5} \\
\midrule
MSE & \makebox[0.38\textwidth][c]{$0.402 \pm 0.008$} & \makebox[0.38\textwidth][c]{$0.395 \pm 0.004$} \\
HL-Gaussian & \makebox[0.38\textwidth][c]{$0.409 \pm 0.005$} & \makebox[0.38\textwidth][c]{$0.401 \pm 0.001$} \\
Quantile Loss & \makebox[0.38\textwidth][c]{$0.407 \pm 0.007$} & \makebox[0.38\textwidth][c]{$0.399 \pm 0.002$} \\
\bottomrule
\end{tabular}
\end{minipage}
\end{table*}

\begin{table*}[p]
\centering
\caption{Extended Beyond results with the Muon optimizer. Single denotes one model; E5 denotes an ensemble of five.}
\label{tab:extended-beyond-results}
\begin{minipage}[t]{0.49\textwidth}
\centering\tiny\setlength{\tabcolsep}{0.5pt}\renewcommand{\arraystretch}{1.0}
\begin{tabular}{llcc}
\multicolumn{4}{c}{5g\_energy\_consumption} \\
\toprule
Model & Objective & Single & E5 \\
\midrule
{MLP} & MSE & $4.07 \pm 0.03$ & $4.02 \pm 0.01$ \\
 & MAE & $4.17 \pm 0.03$ & $4.13 \pm 0.01$ \\
 & Huber Loss & $4.15 \pm 0.05$ & $4.09 \pm 0.02$ \\
 & Soft Binning & $4.13 \pm 0.03$ & $4.09 \pm 0.02$ \\
 & Hard Binning & $4.06 \pm 0.03$ & $4.054 \pm 0.007$ \\
 & Joint End-To-End & $4.10 \pm 0.04$ & $4.06 \pm 0.02$ \\
 & Decoder[N] & $4.48 \pm 0.06$ & $4.22 \pm 0.05$ \\
 & HL-Gaussian & $4.13 \pm 0.04$ & $4.080 \pm 0.004$ \\
 & Decoder[U] & $2753 \pm 10647$ & $2752 \pm 4759$ \\
 & Softmax MSE & $4.14 \pm 0.03$ & $4.08 \pm 0.02$ \\
 & Ordinal Entropy & $4.11 \pm 0.05$ & $4.056 \pm 0.008$ \\
 & ROT-REG & $4.06 \pm 0.04$ & $4.00 \pm 0.03$ \\
 & GAR [MAE] & $4.05 \pm 0.04$ & $4.01 \pm 0.01$ \\
 & GAR [MSE] & $4.08 \pm 0.05$ & $4.020 \pm 0.007$ \\
 & MDN & $4.04 \pm 0.03$ & $4.008 \pm 0.010$ \\
 & NLL & $4.09 \pm 0.04$ & $4.03 \pm 0.01$ \\
 & beta-NLL & $4.08 \pm 0.03$ & $4.039 \pm 0.005$ \\
 & AdaQR & $4.04 \pm 0.01$ & $4.017 \pm 0.002$ \\
 & Quantile Loss & $4.06 \pm 0.03$ & $4.026 \pm 0.006$ \\
\cmidrule(lr){1-4}
{TabPack} & MSE & $4.03 \pm 0.02$ & -- \\
 & HL-Gaussian & $4.06 \pm 0.03$ & -- \\
 & Quantile Loss & $4.01 \pm 0.02$ & -- \\
 & AdaQR & $4.01 \pm 0.02$ & -- \\
\cmidrule(lr){1-4}
{RealMLP} & MSE & $4.04 \pm 0.02$ & $4.027 \pm 0.008$ \\
 & HL-Gaussian & $4.12 \pm 0.03$ & $4.11 \pm 0.02$ \\
 & Quantile Loss & $4.06 \pm 0.02$ & $4.04 \pm 0.01$ \\
 & $\text{AdaQR}^{\dagger}$ & $4.03 \pm 0.03$ & $4.01 \pm 0.02$ \\
 & AdaQR w. global $\sigma$ & $4.03 \pm 0.02$ & $4.01 \pm 0.02$ \\
\cmidrule(lr){1-4}
{TabM} & MSE & $4.01 \pm 0.03$ & $3.99 \pm 0.01$ \\
 & HL-Gaussian & $3.99 \pm 0.02$ & $3.98 \pm 0.01$ \\
 & Quantile Loss & $4.01 \pm 0.04$ & $3.98 \pm 0.02$ \\
 & AdaQR & $4.00 \pm 0.02$ & $3.986 \pm 0.006$ \\
\bottomrule
\end{tabular}
\end{minipage}
\hfill
\begin{minipage}[t]{0.49\textwidth}
\centering\tiny\setlength{\tabcolsep}{0.5pt}\renewcommand{\arraystretch}{1.0}
\begin{tabular}{llcc}
\multicolumn{4}{c}{airfoil\_self\_noise} \\
\toprule
Model & Objective & Single & E5 \\
\midrule
{MLP} & MSE & $1.85 \pm 0.05$ & $1.79 \pm 0.02$ \\
 & MAE & $1.89 \pm 0.07$ & $1.82 \pm 0.02$ \\
 & Huber Loss & $1.91 \pm 0.05$ & $1.85 \pm 0.01$ \\
 & Soft Binning & $1.89 \pm 0.06$ & $1.804 \pm 0.008$ \\
 & Hard Binning & $2.19 \pm 0.02$ & $2.172 \pm 0.008$ \\
 & Joint End-To-End & $1.93 \pm 0.04$ & $1.86 \pm 0.01$ \\
 & Decoder[N] & $2.11 \pm 0.08$ & $1.944 \pm 0.005$ \\
 & HL-Gaussian & $1.90 \pm 0.05$ & $1.84 \pm 0.03$ \\
 & Decoder[U] & $4 \pm 2$ & $3 \pm 1$ \\
 & Softmax MSE & $1.86 \pm 0.06$ & $1.81 \pm 0.04$ \\
 & Ordinal Entropy & -- & -- \\
 & ROT-REG & $1.82 \pm 0.05$ & $1.77 \pm 0.02$ \\
 & GAR [MAE] & -- & -- \\
 & GAR [MSE] & -- & -- \\
 & MDN & $1.89 \pm 0.07$ & $1.825 \pm 0.007$ \\
 & NLL & $1.91 \pm 0.06$ & $1.81 \pm 0.01$ \\
 & beta-NLL & $1.88 \pm 0.06$ & $1.84 \pm 0.03$ \\
 & AdaQR & $1.86 \pm 0.05$ & $1.83 \pm 0.03$ \\
 & Quantile Loss & $1.86 \pm 0.05$ & $1.815 \pm 0.007$ \\
\cmidrule(lr){1-4}
{TabPack} & MSE & $1.70 \pm 0.03$ & -- \\
 & HL-Gaussian & $1.77 \pm 0.02$ & -- \\
 & Quantile Loss & $1.71 \pm 0.03$ & -- \\
 & AdaQR & $1.69 \pm 0.02$ & -- \\
\cmidrule(lr){1-4}
{RealMLP} & MSE & $1.55 \pm 0.02$ & $1.53 \pm 0.01$ \\
 & HL-Gaussian & $1.50 \pm 0.01$ & $1.491 \pm 0.007$ \\
 & Quantile Loss & $1.43 \pm 0.02$ & $1.43 \pm 0.01$ \\
 & $\text{AdaQR}^{\dagger}$ & $1.47 \pm 0.02$ & $1.460 \pm 0.007$ \\
 & AdaQR w. global $\sigma$ & $1.53 \pm 0.01$ & $1.529 \pm 0.009$ \\
\cmidrule(lr){1-4}
{TabM} & MSE & $1.64 \pm 0.03$ & $1.61 \pm 0.01$ \\
 & HL-Gaussian & $1.63 \pm 0.02$ & $1.60 \pm 0.01$ \\
 & Quantile Loss & $1.68 \pm 0.03$ & $1.64 \pm 0.01$ \\
 & AdaQR & $1.62 \pm 0.02$ & $1.60 \pm 0.01$ \\
\bottomrule
\end{tabular}
\end{minipage}
\end{table*}

\begin{table*}[p]
\centering
\caption{Extended Beyond results with the Muon optimizer (continued).}
\begin{minipage}[t]{0.49\textwidth}
\centering\tiny\setlength{\tabcolsep}{0.5pt}\renewcommand{\arraystretch}{1.0}
\begin{tabular}{llcc}
\multicolumn{4}{c}{allstate\_claims\_severity} \\
\toprule
Model & Objective & Single & E5 \\
\midrule
{MLP} & MSE & $0.5400 \pm 0.0002$ & $0.5388 \pm 0.0001$ \\
 & MAE & $0.5408 \pm 0.0001$ & $0.5402 \pm 0.0000$ \\
 & Huber Loss & $0.5393 \pm 0.0002$ & $0.5388 \pm 0.0000$ \\
 & Soft Binning & $0.5378 \pm 0.0001$ & $0.5373 \pm 0.0000$ \\
 & Hard Binning & $0.5393 \pm 0.0001$ & $0.5389 \pm 0.0000$ \\
 & Joint End-To-End & $0.5394 \pm 0.0002$ & $0.5387 \pm 0.0000$ \\
\bottomrule
\end{tabular}
\end{minipage}
\hfill
\begin{minipage}[t]{0.49\textwidth}
\centering\tiny\setlength{\tabcolsep}{0.5pt}\renewcommand{\arraystretch}{1.0}
\begin{tabular}{llcc}
\multicolumn{4}{c}{blood\_tests\_drink\_prediction} \\
\toprule
Model & Objective & Single & E5 \\
\midrule
{MLP} & MSE & $0.667 \pm 0.009$ & $0.666 \pm 0.006$ \\
 & MAE & $0.68 \pm 0.01$ & $0.682 \pm 0.004$ \\
 & Huber Loss & $0.68 \pm 0.02$ & $0.677 \pm 0.005$ \\
 & Soft Binning & $0.667 \pm 0.005$ & $0.666 \pm 0.003$ \\
 & Hard Binning & $0.70 \pm 0.02$ & $0.693 \pm 0.003$ \\
 & Joint End-To-End & $0.671 \pm 0.006$ & $0.670 \pm 0.002$ \\
 & Decoder[N] & $0.86 \pm 0.04$ & $0.76 \pm 0.01$ \\
 & HL-Gaussian & $0.652 \pm 0.008$ & $0.6506 \pm 0.0008$ \\
 & Decoder[U] & $0.74 \pm 0.03$ & $0.699 \pm 0.009$ \\
 & Softmax MSE & $0.669 \pm 0.006$ & $0.668 \pm 0.004$ \\
 & Ordinal Entropy & $0.69 \pm 0.02$ & $0.678 \pm 0.003$ \\
 & ROT-REG & $0.662 \pm 0.010$ & $0.660 \pm 0.005$ \\
 & GAR [MAE] & $0.68 \pm 0.02$ & $0.674 \pm 0.005$ \\
 & GAR [MSE] & $0.674 \pm 0.008$ & $0.672 \pm 0.005$ \\
 & MDN & $0.66 \pm 0.01$ & $0.655 \pm 0.007$ \\
 & NLL & $0.674 \pm 0.008$ & $0.671 \pm 0.005$ \\
 & beta-NLL & $0.674 \pm 0.008$ & $0.668 \pm 0.008$ \\
 & AdaQR & $0.687 \pm 0.008$ & $0.685 \pm 0.002$ \\
 & Quantile Loss & $0.678 \pm 0.006$ & $0.678 \pm 0.004$ \\
\cmidrule(lr){1-4}
{TabPack} & MSE & $0.66 \pm 0.01$ & -- \\
 & HL-Gaussian & $0.668 \pm 0.009$ & -- \\
 & Quantile Loss & $0.654 \pm 0.005$ & -- \\
 & AdaQR & $0.66 \pm 0.01$ & -- \\
\cmidrule(lr){1-4}
{RealMLP} & MSE & $0.74 \pm 0.03$ & $0.693 \pm 0.005$ \\
 & HL-Gaussian & $0.666 \pm 0.005$ & $0.665 \pm 0.002$ \\
 & Quantile Loss & $0.670 \pm 0.008$ & $0.668 \pm 0.002$ \\
 & $\text{AdaQR}^{\dagger}$ & $0.68 \pm 0.01$ & $0.668 \pm 0.006$ \\
 & AdaQR w. global $\sigma$ & $0.666 \pm 0.004$ & $0.666 \pm 0.002$ \\
\cmidrule(lr){1-4}
{TabM} & MSE & $0.671 \pm 0.008$ & $0.670 \pm 0.001$ \\
 & HL-Gaussian & $0.67 \pm 0.01$ & $0.666 \pm 0.004$ \\
 & Quantile Loss & $0.675 \pm 0.007$ & $0.672 \pm 0.003$ \\
 & AdaQR & $0.67 \pm 0.01$ & $0.672 \pm 0.006$ \\
\bottomrule
\end{tabular}
\end{minipage}
\end{table*}

\begin{table*}[p]
\centering
\caption{Extended Beyond results with the Muon optimizer (continued).}
\begin{minipage}[t]{0.49\textwidth}
\centering\tiny\setlength{\tabcolsep}{0.5pt}\renewcommand{\arraystretch}{1.0}
\begin{tabular}{llcc}
\multicolumn{4}{c}{body\_density\_prediction} \\
\toprule
Model & Objective & Single & E5 \\
\midrule
{MLP} & MSE & $0.013 \pm 0.002$ & $0.0127 \pm 0.0002$ \\
 & MAE & $0.0121 \pm 0.0002$ & $0.0119 \pm 0.0001$ \\
 & Huber Loss & $0.0125 \pm 0.0004$ & $0.0123 \pm 0.0000$ \\
 & Soft Binning & $0.0137 \pm 0.0002$ & $0.0137 \pm 0.0001$ \\
 & Hard Binning & $0.0153 \pm 0.0008$ & $0.0143 \pm 0.0004$ \\
 & Joint End-To-End & $0.0136 \pm 0.0002$ & $0.0133 \pm 0.0001$ \\
\bottomrule
\end{tabular}
\end{minipage}
\hfill
\begin{minipage}[t]{0.49\textwidth}
\centering\tiny\setlength{\tabcolsep}{0.5pt}\renewcommand{\arraystretch}{1.0}
\begin{tabular}{llcc}
\multicolumn{4}{c}{california\_house\_prices\_2020} \\
\toprule
Model & Objective & Single & E5 \\
\midrule
{MLP} & MSE & $0.274 \pm 0.002$ & $0.2710 \pm 0.0008$ \\
 & MAE & $0.260 \pm 0.002$ & $0.2573 \pm 0.0008$ \\
 & Huber Loss & $0.267 \pm 0.002$ & $0.258 \pm 0.001$ \\
 & Soft Binning & $0.257 \pm 0.002$ & $0.2531 \pm 0.0006$ \\
 & Hard Binning & $0.306 \pm 0.003$ & $0.294 \pm 0.002$ \\
 & Joint End-To-End & $0.275 \pm 0.002$ & $0.2718 \pm 0.0009$ \\
\bottomrule
\end{tabular}
\end{minipage}
\end{table*}

\begin{table*}[p]
\centering
\caption{Extended Beyond results with the Muon optimizer (continued).}
\begin{minipage}[t]{0.49\textwidth}
\centering\tiny\setlength{\tabcolsep}{0.5pt}\renewcommand{\arraystretch}{1.0}
\begin{tabular}{llcc}
\multicolumn{4}{c}{cirrhosis\_patient\_survival\_prediction} \\
\toprule
Model & Objective & Single & E5 \\
\midrule
{MLP} & MSE & $0.72 \pm 0.02$ & $0.705 \pm 0.005$ \\
 & MAE & $0.66 \pm 0.02$ & $0.655 \pm 0.008$ \\
 & Huber Loss & $0.67 \pm 0.01$ & $0.668 \pm 0.005$ \\
 & Soft Binning & $0.69 \pm 0.01$ & $0.693 \pm 0.006$ \\
 & Hard Binning & $0.93 \pm 0.07$ & $0.83 \pm 0.02$ \\
 & Joint End-To-End & $0.74 \pm 0.01$ & $0.734 \pm 0.002$ \\
 & Decoder[N] & $1.0 \pm 0.2$ & $0.80 \pm 0.08$ \\
 & HL-Gaussian & $0.78 \pm 0.03$ & $0.779 \pm 0.008$ \\
 & Decoder[U] & $0.83 \pm 0.06$ & $0.76 \pm 0.02$ \\
 & Softmax MSE & $0.73 \pm 0.01$ & $0.725 \pm 0.002$ \\
 & Ordinal Entropy & $0.72 \pm 0.02$ & $0.712 \pm 0.005$ \\
 & ROT-REG & $0.76 \pm 0.04$ & $0.73 \pm 0.01$ \\
 & GAR [MAE] & $0.68 \pm 0.01$ & $0.678 \pm 0.003$ \\
 & GAR [MSE] & $0.67 \pm 0.02$ & $0.671 \pm 0.007$ \\
 & MDN & $0.68 \pm 0.02$ & $0.680 \pm 0.006$ \\
 & NLL & $0.70 \pm 0.02$ & $0.69 \pm 0.01$ \\
 & beta-NLL & $0.64 \pm 0.01$ & $0.636 \pm 0.002$ \\
 & AdaQR & $0.680 \pm 0.010$ & $0.679 \pm 0.002$ \\
 & Quantile Loss & $0.68 \pm 0.01$ & $0.681 \pm 0.005$ \\
\cmidrule(lr){1-4}
{TabPack} & MSE & $0.77 \pm 0.03$ & -- \\
 & HL-Gaussian & $0.91 \pm 0.08$ & -- \\
 & Quantile Loss & $0.75 \pm 0.02$ & -- \\
 & AdaQR & $0.72 \pm 0.03$ & -- \\
\cmidrule(lr){1-4}
{RealMLP} & MSE & $0.70 \pm 0.06$ & $0.68 \pm 0.04$ \\
 & HL-Gaussian & $0.78 \pm 0.05$ & $0.74 \pm 0.03$ \\
 & Quantile Loss & $0.658 \pm 0.009$ & $0.656 \pm 0.001$ \\
 & $\text{AdaQR}^{\dagger}$ & $0.70 \pm 0.05$ & $0.69 \pm 0.01$ \\
 & AdaQR w. global $\sigma$ & $0.634 \pm 0.009$ & $0.633 \pm 0.003$ \\
\cmidrule(lr){1-4}
{TabM} & MSE & $0.649 \pm 0.007$ & $0.646 \pm 0.001$ \\
 & HL-Gaussian & $0.67 \pm 0.01$ & $0.671 \pm 0.006$ \\
 & Quantile Loss & $0.68 \pm 0.01$ & $0.676 \pm 0.005$ \\
 & AdaQR & $0.68 \pm 0.01$ & $0.681 \pm 0.001$ \\
\bottomrule
\end{tabular}
\end{minipage}
\hfill
\begin{minipage}[t]{0.49\textwidth}
\centering\tiny\setlength{\tabcolsep}{0.5pt}\renewcommand{\arraystretch}{1.0}
\begin{tabular}{llcc}
\multicolumn{4}{c}{climate\_model\_weather\_forecasting\_1m} \\
\toprule
Model & Objective & Single & E5 \\
\midrule
{MLP} & MSE & $1.474 \pm 0.009$ & $1.468 \pm 0.006$ \\
 & MAE & $1.454 \pm 0.008$ & $1.440 \pm 0.006$ \\
 & Huber Loss & $1.47 \pm 0.01$ & $1.458 \pm 0.004$ \\
 & Soft Binning & $1.470 \pm 0.008$ & $1.460 \pm 0.001$ \\
 & Hard Binning & $1.448 \pm 0.007$ & $1.434 \pm 0.003$ \\
 & Joint End-To-End & $1.46 \pm 0.01$ & $1.443 \pm 0.007$ \\
 & Decoder[N] & $1.51 \pm 0.01$ & $1.455 \pm 0.008$ \\
 & HL-Gaussian & $1.467 \pm 0.008$ & $1.455 \pm 0.004$ \\
 & Decoder[U] & $1.49 \pm 0.01$ & $1.444 \pm 0.005$ \\
 & Softmax MSE & $1.49 \pm 0.01$ & $1.473 \pm 0.005$ \\
 & Ordinal Entropy & $1.50 \pm 0.01$ & $1.481 \pm 0.002$ \\
 & ROT-REG & $1.470 \pm 0.008$ & $1.463 \pm 0.005$ \\
 & GAR [MAE] & $1.451 \pm 0.008$ & $1.442 \pm 0.003$ \\
 & GAR [MSE] & $1.476 \pm 0.009$ & $1.468 \pm 0.002$ \\
 & MDN & $1.45 \pm 0.01$ & $1.437 \pm 0.004$ \\
 & NLL & $1.45 \pm 0.01$ & $1.436 \pm 0.005$ \\
 & beta-NLL & $1.436 \pm 0.008$ & $1.4259 \pm 0.0008$ \\
 & AdaQR & $1.46 \pm 0.01$ & $1.444 \pm 0.008$ \\
 & Quantile Loss & $1.45 \pm 0.01$ & $1.444 \pm 0.005$ \\
\cmidrule(lr){1-4}
{TabPack} & MSE & $1.376 \pm 0.004$ & -- \\
 & HL-Gaussian & $1.373 \pm 0.004$ & -- \\
 & Quantile Loss & $1.365 \pm 0.003$ & -- \\
 & AdaQR & $1.368 \pm 0.003$ & -- \\
\cmidrule(lr){1-4}
{RealMLP} & MSE & $1.402 \pm 0.005$ & $1.389 \pm 0.003$ \\
 & HL-Gaussian & $1.372 \pm 0.003$ & $1.365 \pm 0.002$ \\
 & Quantile Loss & $1.379 \pm 0.005$ & $1.371 \pm 0.002$ \\
 & $\text{AdaQR}^{\dagger}$ & $1.380 \pm 0.006$ & $1.372 \pm 0.002$ \\
 & AdaQR w. global $\sigma$ & $1.370 \pm 0.003$ & $1.364 \pm 0.001$ \\
\cmidrule(lr){1-4}
{TabM} & MSE & $1.394 \pm 0.009$ & $1.388 \pm 0.002$ \\
 & HL-Gaussian & $1.382 \pm 0.005$ & $1.375 \pm 0.005$ \\
 & Quantile Loss & $1.368 \pm 0.009$ & $1.356 \pm 0.003$ \\
 & AdaQR & $1.359 \pm 0.004$ & $1.354 \pm 0.003$ \\
\bottomrule
\end{tabular}
\end{minipage}
\end{table*}

\begin{table*}[p]
\centering
\caption{Extended Beyond results with the Muon optimizer (continued).}
\begin{minipage}[t]{0.49\textwidth}
\centering\tiny\setlength{\tabcolsep}{0.5pt}\renewcommand{\arraystretch}{1.0}
\begin{tabular}{llcc}
\multicolumn{4}{c}{coffee\_rating\_prediction} \\
\toprule
Model & Objective & Single & E5 \\
\midrule
{MLP} & MSE & $1.48 \pm 0.02$ & $1.4771 \pm 0.0003$ \\
 & MAE & $1.48 \pm 0.05$ & $1.46 \pm 0.02$ \\
 & Huber Loss & $1.54 \pm 0.07$ & $1.51 \pm 0.02$ \\
 & Soft Binning & $1.43 \pm 0.04$ & $1.406 \pm 0.008$ \\
 & Hard Binning & $1.53 \pm 0.04$ & $1.53 \pm 0.02$ \\
 & Joint End-To-End & $1.51 \pm 0.02$ & $1.503 \pm 0.008$ \\
 & Decoder[N] & $1.41 \pm 0.05$ & $1.36 \pm 0.02$ \\
 & HL-Gaussian & $1.40 \pm 0.02$ & $1.390 \pm 0.005$ \\
 & Decoder[U] & $1.5 \pm 0.1$ & $1.44 \pm 0.03$ \\
 & Softmax MSE & $1.44 \pm 0.02$ & $1.424 \pm 0.005$ \\
 & Ordinal Entropy & $1.50 \pm 0.03$ & $1.495 \pm 0.007$ \\
 & ROT-REG & $1.48 \pm 0.02$ & $1.465 \pm 0.006$ \\
 & GAR [MAE] & $1.54 \pm 0.06$ & $1.52 \pm 0.02$ \\
 & GAR [MSE] & $1.53 \pm 0.05$ & $1.50 \pm 0.03$ \\
 & MDN & $1.56 \pm 0.02$ & $1.553 \pm 0.003$ \\
 & NLL & $1.49 \pm 0.05$ & $1.47 \pm 0.02$ \\
 & beta-NLL & $1.54 \pm 0.03$ & $1.5312 \pm 0.0010$ \\
 & AdaQR & $1.50 \pm 0.07$ & $1.48 \pm 0.03$ \\
 & Quantile Loss & $1.57 \pm 0.05$ & $1.56 \pm 0.01$ \\
\cmidrule(lr){1-4}
{TabPack} & MSE & $1.34 \pm 0.02$ & -- \\
 & HL-Gaussian & $1.36 \pm 0.02$ & -- \\
 & Quantile Loss & $1.33 \pm 0.02$ & -- \\
 & AdaQR & $1.34 \pm 0.02$ & -- \\
\cmidrule(lr){1-4}
{RealMLP} & MSE & $1.31 \pm 0.04$ & $1.291 \pm 0.010$ \\
 & HL-Gaussian & $1.28 \pm 0.02$ & $1.276 \pm 0.009$ \\
 & Quantile Loss & $1.27 \pm 0.01$ & $1.269 \pm 0.007$ \\
 & $\text{AdaQR}^{\dagger}$ & $1.259 \pm 0.007$ & $1.252 \pm 0.003$ \\
 & AdaQR w. global $\sigma$ & $1.32 \pm 0.02$ & $1.316 \pm 0.005$ \\
\cmidrule(lr){1-4}
{TabM} & MSE & $1.33 \pm 0.02$ & $1.311 \pm 0.006$ \\
 & HL-Gaussian & $1.34 \pm 0.03$ & $1.33 \pm 0.01$ \\
 & Quantile Loss & $1.29 \pm 0.04$ & $1.28 \pm 0.02$ \\
 & AdaQR & $1.30 \pm 0.02$ & $1.289 \pm 0.009$ \\
\bottomrule
\end{tabular}%
\end{minipage}
\end{table*}

\end{document}